\documentclass[nohyperref]{article}

\usepackage{microtype}
\usepackage{graphicx}
\usepackage{subfigure}
\usepackage{booktabs} 

\usepackage{hyperref}

\usepackage[accepted]{icml2022}

\usepackage{amsmath}
\usepackage{amssymb}
\usepackage{mathtools}
\usepackage{amsthm}

\usepackage{graphicx}
\usepackage{subfigure}
\usepackage{booktabs} 

\usepackage{dsfont}
\usepackage{multirow}
\usepackage{float,paralist}
\usepackage{transparent}
\usepackage{stfloats}
\usepackage{url,wrapfig}

\mathtoolsset{showonlyrefs}

\newcommand{\R}{\mathbb{R}}
\newcommand{\norm}[1]{\left\|#1\right\|}
\DeclarePairedDelimiterX{\inner}[1]{\langle}{\rangle}{#1}
\newcommand{\supp}{\text{supp}}

\theoremstyle{plain}
\newtheorem{theorem}{Theorem}[section]
\newtheorem{proposition}[theorem]{Proposition}
\newtheorem{lemma}[theorem]{Lemma}
\newtheorem{observation}[theorem]{Observation}

\newtheorem*{lemma*}{Lemma}
\newtheorem*{proposition*}{Proposition}

\theoremstyle{definition}

\theoremstyle{remark}
\newtheorem{remark}[theorem]{Remark}

\usepackage[textsize=tiny]{todonotes}

\icmltitlerunning{Multi Resolution Analysis (MRA) for Approximate Self-Attention}

\begin{document}

\twocolumn[
\icmltitle{Multi Resolution Analysis (MRA) for Approximate Self-Attention}



\icmlsetsymbol{equal}{*}

\begin{icmlauthorlist}
\icmlauthor{Zhanpeng Zeng}{wisc}
\icmlauthor{Sourav Pal}{wisc}
\icmlauthor{Jeffery Kline}{amfam}
\icmlauthor{Glenn Fung}{amfam}
\icmlauthor{Vikas Singh}{wisc}
\end{icmlauthorlist}

\icmlaffiliation{wisc}{University of Wisconsin, Madison, USA}
\icmlaffiliation{amfam}{American Family Insurance, Madison, USA}

\icmlcorrespondingauthor{Zhanpeng Zeng}{zzeng38@wisc.edu}

\icmlkeywords{Machine Learning, ICML}

\vskip 0.3in
]


\begin{figure*}[!hbp]
\centering
\vspace{-0.16in}
\includegraphics[width=6.5in]{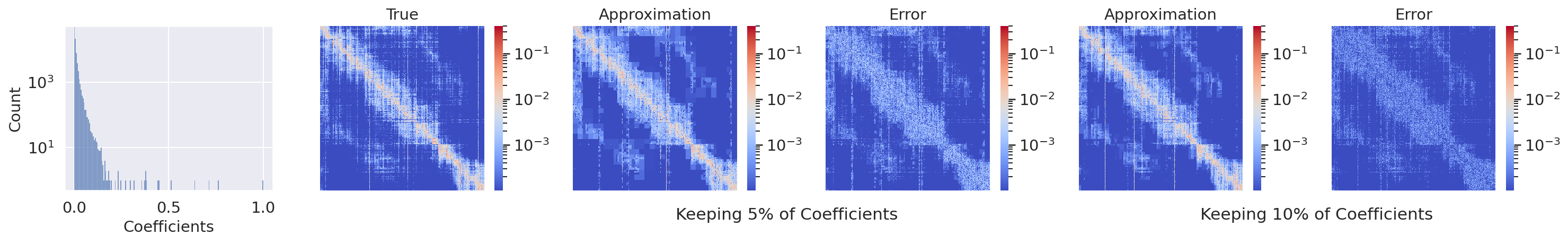}
\vspace{-0.15in}
\caption{The left plot is the histogram of wavelet coefficients in log scale for 2D Haar basis. More than $95$\% of coefficients have a magnitude less than $0.005$. The second plot is the true self-attention matrix $\mathcal{A}$. The third and fourth plots are the approximation and error by keeping top $5$\% of coefficients. The fifth and last plots are similar for keeping top $10$\% of coefficients. The errors $||\hat{\mathcal{A}} - \mathcal{A}||_F$ for $\{$MRA, low rank, sparsity$\}$ by keeping 10\% of $\{$coefficients, ranks, nonzero entries$\}$, are $0.30, 1.24, 0.39$, respectively. This shows  performance benefits of MRA compared to low rank and sparsity. We discuss low rank and sparsity in more detail in \S\ref{sec:sparse_low_rank_link}. }
\label{fig:haar_approx}
\vspace{-0.1in}
\end{figure*}


\printAffiliationsAndNotice{}  

\begin{abstract}

Transformers have emerged as a preferred
model for many tasks in natural langugage processing and vision. Recent efforts on training and deploying 
Transformers more 
efficiently have identified many strategies 
to approximate the self-attention 
matrix, a key module in a Transformer architecture.
Effective ideas include various prespecified sparsity patterns, low-rank basis expansions and combinations thereof. 
In this paper, we revisit classical Multiresolution 
Analysis (MRA) concepts such as Wavelets, whose 
potential value in this setting remains underexplored 
thus far. We show that simple 
approximations based on empirical feedback and design choices informed by modern hardware and implementation challenges, eventually yield a MRA-based 
approach for self-attention with an excellent 
performance profile across most criteria of interest. 
We undertake an extensive set of experiments and demonstrate that this 
multi-resolution scheme outperforms 
most efficient self-attention proposals and is favorable for both short and long sequences. Code is available at \url{https://github.com/mlpen/mra-attention}.

\vspace{-0.11in}


\end{abstract}

\section{Introduction}

\label{sec:introduction}

The Transformer model proposed in  \cite{vaswani2017attention} 
is the architecture of choice for a number of tasks 
in natural language processing (NLP) as well 
as 
vision. 
The ``self-attention'' mechanism within Transformers, 
and specifically, its extension referred to as Multi-Head 
Self-Attention, enable 
modeling dynamic token dependencies, 
and play a key role in the overall 
performance profile of the model. 
Despite these advantages, the $O(n^2)$ complexity of 
the self-attention
(where $n$ is the sequence length) 
is a bottleneck, especially for longer sequences. 
Consequently, over the last year or so, 
extensive effort has been devoted 
in deriving methods that mitigate this quadratic cost -- 
a wide choice of algorithms with 
linear (in the sequence length) complexity are now available \cite{Wang2020LinformerSW, choromanski2020rethinking, yoso, beltagy2020longformer, zaheer2020big}.

The aforementioned body of work 
on {\em efficient} transformers leverages 
the observation that the self-attention matrix has a 
parsimonious representation -- the mechanics of how this 
is modeled and exploited at the algorithm level, varies 
from one method to the other. 
For instance, we may ask that 
self-attention has a pre-specified form of sparsity (for instance, diagonal or banded), see \cite{beltagy2020longformer, zaheer2020big}. Alternatively, we may 
model self-attention globally as a low-rank matrix, successfully utilized in \cite{Wang2020LinformerSW, xiong2021nystromformer, choromanski2020rethinking}. Clearly, 
each modeling choice entails its own form of 
approximation error, which we can measure empirically 
on existing benchmarks. 
Progress has been brisk in improving these approximations:  recent proposals have investigated 
a hybrid global+local strategy based on invoking a robust 
PCA style model \cite{chen2021scatterbrain} and hierarchical (or H-) matrices \cite{htransformer1d}, 
which we will draw a contrast with. 

{\bf MRA.} Recall that the hybrid global+local intuition above  
has a rich classical treatment 
formalized under Multiresolution analysis (MRA) methods,
and Wavelets are a prominent example \cite{mallat1999wavelet}. 
While the use of wavelets for signal processing goes 
back at least three decades (if not longer), 
their use in machine 
learning especially for graph based datasets,  
as a numerical preconditioner, and even for 
matrix decomposition problems 
has seen a resurgence  \cite{ kondor2014multiresolution, ithapu2017incremental,gavish2010multiscale, lee2007treelets, hammond2011wavelets, coifman2006diffusion}. 
The extent to 
which classical MRA ideas (or even their heuristic 
forms) can guide efficiency in 
Transformers is largely unknown. Fig. \ref{fig:haar_approx} shows that, given 
a representative 
self-attention matrix, only $10\%$ of 
the coefficients are sufficient for a high 
fidelity reconstruction. 
At a minimum, we see that the hypothesis of 
evaluating a MRA-based self-attention within a Transformer 
model may have merit. 

{\bf Contributions.} The {\bf goal}
of this paper is to investigate the specific modifications, 
adjustments and approximations needed to make the 
MRA idea 
operational in Transformers. It turns out that modulo some small compromises (on the theoretical side), MRA-based self-attention shows  
excellent performance across the board -- it is competitive with standard self-attention and outperforms most of baselines while maintaining significantly high time and memory efficiency on both short and long sequences. 



\section{Preliminaries: Self-attention and Wavelets}

We briefly review self attention and wavelet decomposition, two concepts that we will use throughout the paper.

\subsection{Self-attention}

Given embedding matrices $Q, K, V \in \R^{n \times d}$ representing $n$ $d$-dimensional feature vectors for queries, key and values, respectively, self-attention is defined as
\begin{equation}
\begin{split}
Z = \text{Softmax}\left( \underbrace{Q K^T}_{\mathcal{P}} \right) V = D \underbrace{\exp(\mathcal{P})}_{\mathcal{A}} V
\end{split}
\label{eq:softmax-attention}
\end{equation}
where $D$ is a $n\times n$ diagonal matrix that normalizes each row of the $\mathcal{A}$ matrix such that the row entries sum up to $1$. 
For notational simplicity, the scaling factor $\sqrt{d}$ and linear projections applied to $Q, K, V$ are omitted. An explicit calculation of $\mathcal{A}$ and taking its product with $V$ incurs a $O(n^2)$ cost (if $d$ is treated as fixed/constant for complexity analysis purposes), a serious resource bottleneck and the core 
focus of the existing literature on efficient 
Transformers.

\textbf{Related Work on Efficient Transformers.} A number of efficient self-attention methods have been proposed to reduce the $O(n^2)$ cost. Much of this literature can be, roughly speaking, divided into two
categories: low rank and sparsity. 
Linformer \cite{Wang2020LinformerSW} shows that self-attention matrices are low rank and proposes to learn a projection -- projecting the sequence length dimension to lower dimensions. 
Performer \citep{choromanski2020rethinking} and Random Feature Attention \citep{peng2021rfa} view self-attention matrices as kernel matrices of  infinite feature maps and propose using finite random feature maps to approximate the kernel matrices. 
Nystr\"{o}mformer \cite{xiong2021nystromformer} and SOFT \cite{SOFT} use a 
Nystr\"{o}m method for matrix approximation to approximate the self-attention matrices. 

A number of methods also leverage the sparsity of self-attention matrices. 
By exploiting a high dependency within a local context, Longformer \cite{beltagy2020longformer} proposes a sliding window attention with global attention on 
manually selected tokens. 
In addition to the sparsity used in Longformer, Big Bird \cite{zaheer2020big} adds random sparse attention to further refine the approximation. 
Instead of a prespecified sparsity support, Reformer \cite{Kitaev2020ReformerTE} uses hashing (LSH) to compute self-attention only within approximately nearby tokens, 
and 
YOSO \cite{yoso} uses the collision probability of LSH as attention weights and then, a LSH-based sampler to achieve linear complexity.

The discussion in \cite{chen2021scatterbrain} suggests that approximations relying solely on low rank or sparsity are limited and a hybrid model via robust PCA offers better approximation. Scatterbrain \cite{chen2021scatterbrain} uses a sparse attention + low rank attention strategy to avoid the cost of robust PCA. In \S\ref{sec:sparse_low_rank_link}, we discuss some limitations of low rank and sparsity for self-attention approximation, and show that a special form of our MRA approximation can offer a good solution for a relaxation of robust PCA. 


Independent of our work, recently, H-Transformer-1D \cite{htransformer1d} proposed a hierarchical self-attention where the self-attention matrices have a low rank structure on the off-diagonal entries and attention is precisely calculated for the on-diagonal entries. This is also a form of multiresolution approximation for self-attention although the lower resolution for distant tokens may limit its ability to capture precise long range dependencies. 
While a prespecified structure can indeed provide an effective approximation scheme in specific settings, it would be desirable to avoid restriction to a fixed structure, if possible.

\subsection{Wavelet Transform}
A wavelet transform 
decomposes a signal into different scales and locations represented by a set of scaled and translated copies of a {\em fixed} function. This fixed function $\varphi$ is called a mother wavelet, and the scaled and translated copies are called child wavelets specified by two factors, scale $s$ and translation $t$. 
\begin{equation}
\varphi^{s}_{t}(x) = \frac{1}{\sqrt{s}} \varphi(\frac{x - t}{s})
\end{equation}
Here, $s$ controls the ``dilation'' or inverse of the frequency of wavelet, while $t$ controls the location (e.g., time). These scaled/translated versions of mother wavelets play a key role in MRA. Given a choice of $\varphi$, 
the wavelet transform maps a function $f$ to coefficients 
\begin{equation}
\alpha^{s}_{t} = \langle f, \varphi^{s}_{t} \rangle = \int f(x) \varphi^{s}_{t}(x) dx
\end{equation}
The coefficient $\alpha^{s}_{t}$ 
captures the measurement of $f$ at scale $s$ and location $t$. 


\section{MRA view of Self-attention}

\label{sec:mra}

To motivate the use of MRA, in 
\S\ref{sec:introduction}, we used Fig. \ref{fig:haar_approx} to check 
how a 2D Haar wavelet basis decomposes the target matrix $\mathcal{A}$ 
into terms involving different scales and translations, and terms with larger coefficients suffice for a
good approximation of $\mathcal{A}$. 
But the reader will notice that the calculation of the coefficients requires access to the full matrix $\mathcal{A}$. 
Our discussion below will start from a formulation which will still need the full matrix $\mathcal{A}$. Later, in \S\ref{sec:practical_approx}, by exploiting the locality of $Q$ and $K$, we will be able to derive an approximation with reduced complexity (without access to $\mathcal{A}$). 




For simplicity, we assume that the sequence length $n = 2^k$ for some integer $k$. Inspired by the Haar basis and its ability to adaptively approximate while preserving locality, we apply a pyramidal MRA scheme. 
We consider a decomposition of $\mathcal{A}$ using a set of simpler unnormalized components $B^{s}_{x, y} \in \R^{n \times n}$ defined as 
\begin{equation}
(B^{s}_{x, y})_{i, j} = 
\begin{cases}
    1  & \text{if } sx - s < i \leq sx, sy - s < j \leq sy \\
    0  & \text{otherwise}
\end{cases}
\label{eq:component_matrix}
\end{equation}
for $s \in \{1, 2, 4, \cdots, n\}$  and $x, y \in \{1, \cdots, n / s\}$. Here, $(sx - s, sx] \times (sy - s, sy]$ is the support of $B^{s}_{x, y}$, and $s$ represents the scale of the components, i.e., a smaller $s$ denotes higher resolutions and vice versa. Also, $x, y$ denote the translation of the components. 

{\bf Why not Haar basis?} The main reason for using the form in 
\eqref{eq:component_matrix} instead of a 2D Haar basis directly 
is implementation driven, and will be discussed shortly in Remark \ref{remark:reason_not_use_haar}.
For the moment, we can observe that \eqref{eq:component_matrix} is an overcomplete frame for $\mathbb{R}^{n \times n}$. As shown in Fig. \ref{fig:compoent_compare}, frame \eqref{eq:component_matrix} has {\em one extra} scale (with support on a single entry) compared to the Haar basis (4 rows versus 3 rows). Except for this extra scale, \eqref{eq:component_matrix} has the same support as the Haar basis at different scales. In addition, \eqref{eq:component_matrix} provides scaled and translated copies of the ``mother'' component, similar to Haar. 

\begin{figure}[!htbp]
\centering
\begin{subfigure}
    \centering
    \includegraphics[width=1.2in]{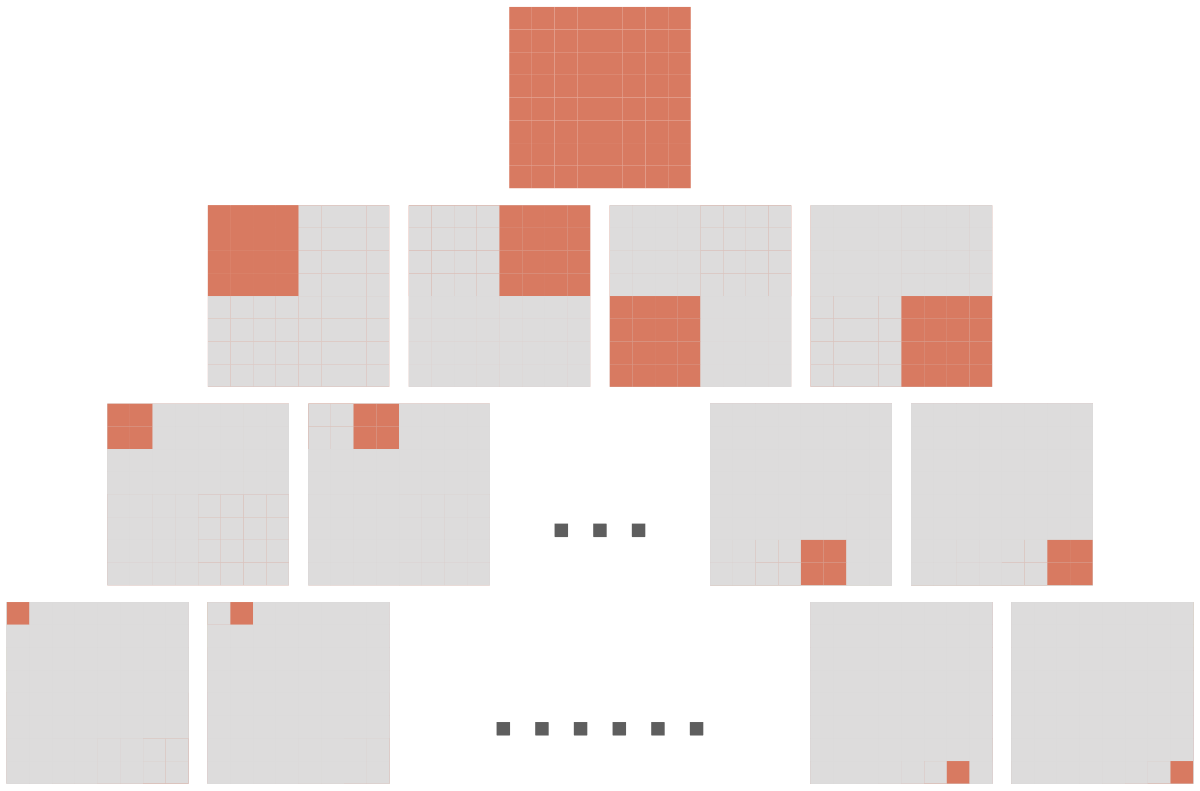}
\end{subfigure}
\hspace{0.25in}
\begin{subfigure}
    \centering
    \includegraphics[width=1.2in]{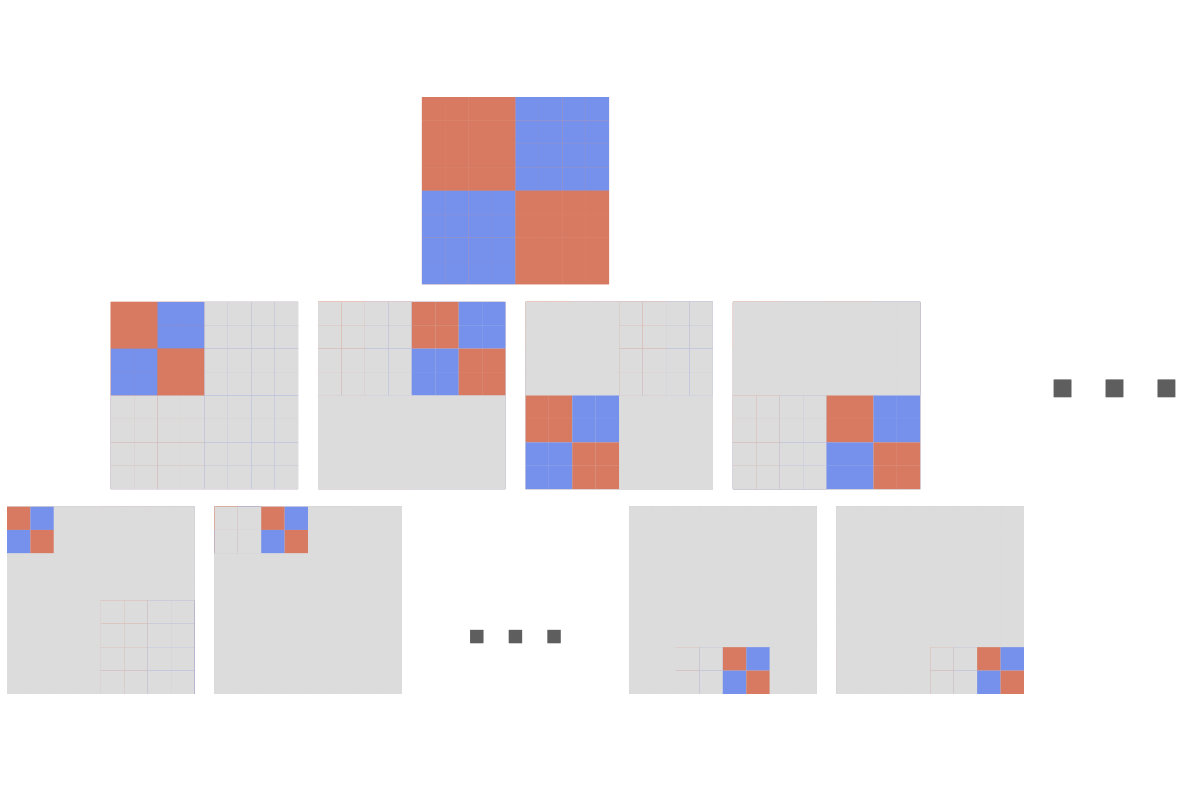}
\end{subfigure}
\vspace{-0.1in}
\caption{The left plot is the overcomplete frame defined in \eqref{eq:component_matrix}, which consists of $85$ matrices for $n = 8$. The right plot is a 2D generalization of Haar basis,
which consists of three groups of $21$ self-similar matrices 
plus a constant matrix (all entries equal to $\frac{1}{n}$). 
The color red, blue, and gray means positive, negative, and $0$, respectively. Notice that \eqref{eq:component_matrix} does not include negative entries: it involves more components but makes the formulation simpler.}
\vspace{-0.15in}
\label{fig:compoent_compare}
\end{figure}

Let $\mathcal{I} = \{B^{s}_{x, y}\}$ be a set of components for the possible scales and translations, then we decompose
\begin{equation}
\mathcal{A} = \sum_{B^{s}_{x, y} \in \mathcal{I}} \alpha^{s}_{x, y} B^{s}_{x, y}
\label{eq:decomposition_of_A}
\end{equation}
for some set of coefficients $\alpha^{s}_{x, y}$. Since \eqref{eq:component_matrix} is overcomplete, the coefficients $\alpha^{s}_{x, y}$ are not unique. We specifically compute the coefficients $\alpha^{s}_{x, y}$ as follows. Let $E_{n} = \mathcal{A}$ and
\begin{equation}
\begin{split}
\quad
\alpha^{s}_{x, y} &= \frac{1}{s^2} \inner{B^{s}_{x, y}, E_{s}} \\
E_{s/2} &= E_{s} - \sum_{B^{s}_{x, y} \in \mathcal{I}} \alpha^{s}_{x, y} B^{s}_{x, y}
\end{split}
\label{eq:decomposition_procedure}
\end{equation}
Here, the $E_{s/2}$ denotes the residuals of the higher frequencies. 
At each scale $s$, $\alpha^{s}_{x, y}$ is the optimal solution of the least squares problem minimizing $||E_{s/2}||_F^2$.  Intuitively, the approximation procedure starts from the coarsest approximation of $\mathcal{A}$ which only consists of the lowest frequency, then the procedure refines the approximation by adding residuals of higher frequencies. 

{\bf Parsimony.} We empirically observe that the coefficients of most components are near zero, so we can represent $\mathcal{A}$ with only a few components while maintaining the accuracy of approximation. Specifically, we can find a small subset $\mathcal{J} \subset \mathcal{I}$, the corresponding coefficients $(\alpha^{s}_{x, y})'$ computed following \eqref{eq:decomposition_procedure}, and the resulting approximation
\begin{equation}
\hat{\mathcal{A}}^* = \sum_{B^{s}_{x, y} \in \mathcal{J}} (\alpha^{s}_{x, y})' B^{s}_{x, y}
\label{eq:approximation_of_A}
\end{equation}
with a negligible approximation error $||\hat{\mathcal{A}}^* - \mathcal{A}||_F$. 

%
But $\hat{\mathcal{A}}^*$ in \eqref{eq:approximation_of_A} does not suggest any interesting property of the approximation. But we can check an equivalent form of $\hat{\mathcal{A}}^*$. Denote the average of $\mathcal{A}$ over the support of $B^{s}_{x, y}$ to be
\begin{equation}
\begin{split}
\mu_{s, x, y}^* &= \frac{1}{s^2} \inner{B^{s}_{x, y}, \exp(\mathcal{P})} \\
\end{split}
\label{eq:average_of_exp}
\end{equation}
It turns out that the entries of $\hat{\mathcal{A}}^*$ in \eqref{eq:approximation_of_A} can be rewritten as
\begin{equation}
\begin{split}
(\hat{\mathcal{A}}^*)_{i,j} &= \mu_{s, x, y}^* \\
\end{split}
\label{eq:approximation_result}
\end{equation}
where $(s, x, y)$ is the index of $B^{s}_{x, y} \in \mathcal{J}$ that has the smallest support region and is supported on $(i, j)$. But if $(i, j) \not\in \supp(B^{s}_{x, y})$ for all $B^{s}_{x, y} \in \mathcal{J}$, then $(\hat{\mathcal{A}}^*)_{i,j} = 0$. The $(i, j)$ entry of $\hat{\mathcal{A}}^*$ is precisely approximated by the average of $\mathcal{A}$ over the smallest support region for $B^{s}_{x, y} \in \mathcal{J}$ containing $(i, j)$. In other words, the procedure uses the highest resolution possible for a given $\mathcal{J}$ as an approximation. We discuss and show how we obtain \eqref{eq:approximation_result} from \eqref{eq:approximation_of_A} in \S\ref{sec:appendix_analysis}.
The reader will notice that rewriting $\hat{\mathcal{A}}^*$ as \eqref{eq:approximation_result} is possible due to our modifications to the Haar basis in \eqref{eq:component_matrix}.

\begin{remark}
\label{remark:reason_not_use_haar}
Consider using a Haar decomposition and let $\mathcal{L}$ be the subset of basis with nonzero coefficients. The approximation $(\hat{\mathcal{A}}_{\rm Haar})_{i,j}$ depends on {\em all} $\phi \in \mathcal{L}$ which are supported on $(i,j)$. For example, in the worst case, the coefficients of all $\phi$ which are supported on $(i,j)$ need to be nonzero to have $(\hat{\mathcal{A}}_{\rm Haar})_{i,j} = \mathcal{A}_{i,j}$. We find that a hardware friendly and efficient approximation scheme in this case is challenging. On the other hand, when using the decomposition \eqref{eq:decomposition_procedure} over the overcomplete frame \eqref{eq:component_matrix},  $(\hat{\mathcal{A}}^*)_{i,j}$ depends on 
only {\em one} $B^s_{x,y} \in \mathcal{J}$ that has the smallest support region and is supported on $(i,j)$. This makes constructing the set $\mathcal{J}$ easier and more flexible.  
\end{remark}

\section{A Practical Approximation scheme}

\label{sec:practical_approx}

Given that we now understand all 
relevant modules, we can focus on  
practical considerations.  
Notice that each $\mu_{s, x, y}^*$ requires averaging over $O(s^2)$ entries of the matrix $\mathcal{A}$, so in the worst case, we would need access to the entire matrix $\mathcal{A}$ to compute all $\mu_{s, x, y}^*$ for $B^s_{x,y} \in \mathcal{J}$. 
Nonetheless, suppose that we still compute all the 
coefficients $\alpha^s_{x,y}$ and then 
post-hoc truncate the small coefficients 
to construct the set $\mathcal{J}$. This approach will clearly 
be inefficient. In this section, we discuss two strategies where the main goal is efficiency.


\subsection{Can we approximate $\mu_{s, x, y}^*$ quickly?}

We first discuss calculating $\mu_{s, x, y}^*$. To avoid accessing the full matrix $\mathcal{A}$, instead of computing the average of exponential \eqref{eq:average_of_exp}, we compute a lower bound (due to convexity of exponential), i.e., exponential of average \eqref{eq:exp_of_average}, as an approximation. 
\begin{equation}
\mu_{s, x, y} = \exp\left(\frac{1}{s^2} \inner{B^{s}_{x, y}, \mathcal{P}}\right)
\label{eq:exp_of_average}
\end{equation}
We can verify that the expression in \eqref{eq:exp_of_average} can be computed efficiently as follows. Define $\tilde{Q}_s, \tilde{K}_s \in \mathbb{R}^{n/s \times d}$ where $\tilde{Q}_1 = Q$, $\tilde{K}_1 = K$, and
\begin{equation}
\begin{split}
(\tilde{Q}_s)_{i} &= \frac{1}{2} (\tilde{Q}_{s/2})_{2i - 1} + \frac{1}{2} (\tilde{Q}_{s/2})_{2i} \\
(\tilde{K}_s)_{i} &= \frac{1}{2} (\tilde{K}_{s/2})_{2i - 1} + \frac{1}{2} (\tilde{K}_{s/2})_{2i}
\end{split}
\label{eq:recursive_pooling}
\end{equation}
Here, $(Q_s)_i$ and $(K_s)_i$ denote the $i$-th row of the matrix $Q_s$ and $K_s$, respectively. Interestingly, \eqref{eq:exp_of_average} is simply,
\begin{equation}
\mu_{s, x, y} = \exp((\tilde{Q}_s)_{x} (\tilde{K}_s)_{y}^T)
\end{equation}
Then, the approximation using \eqref{eq:exp_of_average} is,
\begin{equation}
\begin{split}
\hat{\mathcal{A}}_{i,j} &= \mu_{s, x, y} \\
\end{split}
\label{eq:actual_approx_result}
\end{equation}
where $(s, x, y)$ is the same as \eqref{eq:approximation_result} and $\hat{\mathcal{A}}_{i,j} = 0$ otherwise. 
Each $\mu_{s,x,y}$ only requires an inner product between one row of $\tilde{Q}_s$ and one row of $\tilde{K}_s$ and applying an exponential, so the cost of a single $\mu_{s,x,y}$ is $O(1)$ when $\tilde{Q}_s$ and $\tilde{K}_s$ are provided. We will discuss the overall complexity in \S\ref{sec:complexity}. 

While efficient, this modification will incur a small amount of error. However, by using the property of $\mathcal{P}$ inherited from $Q$ and $K$, we can quantify the error. 
\begin{lemma}
\label{lem:error_bound_of_mu}
Assume for all $(i_1, j_1), (i_2, j_2) \in \supp(B^{s}_{x,y})$, $||Q_{i_1}||_p, ||Q_{i_2}||_p, ||K_{j_1}||_p, ||K_{j_2}||_p \leq \beta_1$ and $||Q_{i_1} - Q_{i_2}||_q, ||K_{j_1} - K_{j_2}||_q \leq \beta_2$ where $\frac{1}{p} + \frac{1}{q} = 1$, then $\mathcal{P}_{i,j} \in [a, a + r]$ where $r = 2 \beta_1 \beta_2$ for all $(i, j) \in \supp(B^{s}_{x,y})$ and some $a$, and 
\begin{equation*}
0 \leq \mu_{s, x, y}^* - \mu_{s, x, y} \leq C_r \mu_{s, x, y}
\end{equation*}
where $C_r = 1 + \exp(r) - 2 \exp(r/2)$. 
\end{lemma}
Lemma \ref{lem:error_bound_of_mu} suggests that the approximation error depends on the ``spread'' or numerical range $r$ of values (range, for short) in the $\mathcal{P}_{i,j}$ entries within a region $\supp(B^{s}_{x,y})$ and $\mu_{s, x, y}$. If $r$ is small or $\mu_{s, x, y}$ is small, then the approximation error is small. The range of a region is influenced by properties of $Q$ and $K$. 
The range $r$ is bounded by the norm and spread of $Q_i$ and $K_j$ for $(i, j) \in \supp(B^{s}_{x,y})$. This relies on the locality assumption that spatially nearby tokens should also be semantically similar which commonly holds in many applications. Of course, this can be avoided if needed -- it is easy to reduce the spread of $Q_i$ and $K_j$ in local regions simply by permuting the order of $Q$ and $K$. For example, we can use Locality Sensitive Hashing (LSH) to reorder $Q_i$ and $K_j$ such that similar vectors are in nearby positions, e.g., see \cite{Kitaev2020ReformerTE}. 
While the range $r$ is data/problem dependent, we can control the range by using a smaller $s$ since the range of a smaller region will be smaller. In the extreme case, when $s = 1$, the range is $0$. So, this offers guidance that when $\mu_{s, x, y}$ is large, we should approximate the region at a higher resolution such that the range is smaller.

\begin{remark}
Observe that the numerical range, which is defined as a bound on finite differences over sets of indices, is closely related to the concept of smoothness, which is defined using finite differences amongst adjacent indices.  Indeed, it is possible to adapt  Lemma \ref{lem:error_bound_of_mu} and its proof to the theory of wavelets, which are useful for characterizing signal smoothness. Please see \S\ref{sec:wavelets} for more details. 
\end{remark}

\begin{remark}
The underlying assumption of diagonal attention structure from Longformer, Big Bird, and H-Transformer-1D is that tokens are highly dependent on the nearby tokens and {\em only} the nearby tokens, which is more important than attention w.r.t. distant tokens. This might appear similar to the locality assumption discussed earlier, but this is incorrect. Our locality does {\em not} assume that semantically similar tokens must be spatially close, i.e., we allow high and precise dependence on distant tokens. 
\label{remark:locality}
\end{remark}

\subsection{Can we construct $\mathcal{J}$ quickly?}

\begin{figure*}[!htbp]
\centering
\vspace{-0.05in}
\includegraphics[width=6.5in]{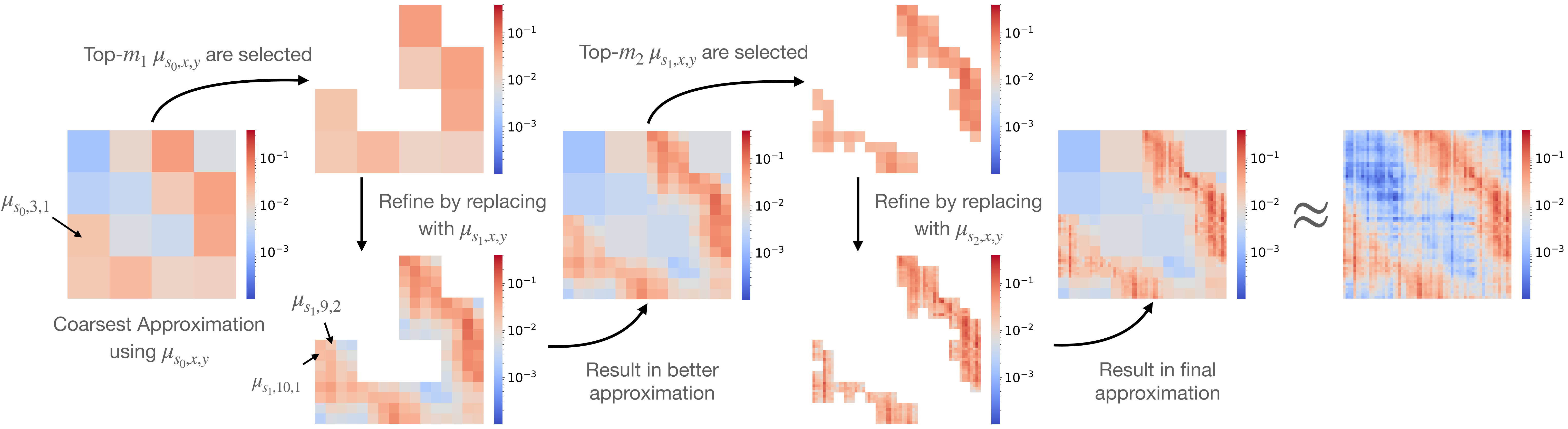}
\vspace{-0.1in}
\caption{Illustration of our approximation scheme for $R = \{16, 4, 1\}$. A log scale is used for a better visualization. A linear scale visualization is shown in \S\ref{sec:approx_procedure_linear}. }
\label{fig:multires_approx}
\vspace{-0.17in}
\end{figure*}


So far, we have assumed that the set $\mathcal{J}$ is given which is not true in practice. 
We now place a mild restriction on the set $\mathcal{J}$ as a simplification heuristic. We allow each $(i, j)$ entry of $\hat{\mathcal{A}}$ to be included in the support of exactly one $B^{s}_{x, y} \in \mathcal{J}$. Consider a $B^{s}_{x, y} \in \mathcal{J}$, if each entry of the support region of $B^{s}_{x, y}$ is included in the support of some $B^{s'}_{x', y'} \in \mathcal{J}$ with a smaller $s'$, then $B^{s}_{x,y}$ can be safely removed from $\mathcal{J}$ without affecting the approximation,  by construction. This restriction allows us to avoid searching for the $B^{s}_{x, y}$ with the smallest $s$ among multiple candidates. Then, the overall approximation can be written as
\begin{equation}
\hat{\mathcal{A}} = \sum_{B^{s}_{x,y} \in \mathcal{J}} \mu_{s,x,y} B^s_{x,y}
\label{eq:actual_approx_mat_form}
\end{equation}

\begin{remark}
Under this restriction, $\mathcal{J}$ is a subset of an orthogonal basis of $\mathbb{R}^{n \times n}$. 
\end{remark}

{\bf Mechanics of constructing $\mathcal{J}$.} Now, we can discuss how the set $\mathcal{J}$ is constructed. {Let us first consider the approximation error $\mathcal{E} = ||\hat{\mathcal{A}} - \mathcal{A}||_F^2$, by factoring out $\mu_{s, x, y}^2$, 
\begin{equation}
\begin{split}
\mathcal{E} &= \sum_{B^{s}_{x,y} \in \mathcal{J}} \mu_{s, x, y}^2  \sum_{(i,j) \in \supp(B^{s}_{x,y})} \left(\frac{\mathcal{A}_{i,j}}{\mu_{s, x, y}} - 1\right)^2
\end{split}
\label{eq:approx_error}
\end{equation}}
Since the goal is to minimize the error $\mathcal{E}$, the optimal solution is to fix the computation budget $|\mathcal{J}|$ and solve the optimization problem which minimizes $\mathcal{E}$ over all possible $\mathcal{J}$. However, this might not be efficiently solvable. 

Instead, we consider finding a good solution greedily. Consider the error $\mathcal{E}$. We can analyze the specific term to get an insight into how approximation error can be reduced. Note that
\begin{equation}
\log \frac{\mathcal{A}_{i,j}}{\mu_{s, x, y}} = \mathcal{P}_{i,j} - \frac{1}{s^2} \inner{B^{s}_{x,y}, \mathcal{P}}
\end{equation}
is the deviation of $\mathcal{P}_{i,j}$ from the mean of the support region, so the approximation error \eqref{eq:approx_error} is determined by $\mu_{s, x, y}$ and the deviation of $\mathcal{P}_{i,j}$ within the region, which coincides with the conclusion of Lemma \ref{lem:error_bound_of_mu}. Computing this deviation would incur a $O(s^2)$ cost, so we avoid using it as a criteria for $\mathcal{J}$ construction. We found that we can make a reasonable assumption that the deviation of $\mathcal{P}_{i,j}$ in a support of $B^{s}_{x,y}$ for the {\em same} $s$ are similar, and the 
deviation of a region for a smaller $s$ is smaller. Then, a sensible heuristic is to use $\mu_{s, x, y}$ as a criteria such that if $\mu_{s, x, y}$ is large, then we must approximate the region using a higher resolution. The approximation procedure is described in Alg. \ref{alg:algorithm}, and the approximation result is shown in Fig. \ref{fig:multires_approx}. Broadly, this approximation starts with a coarse approximation of a self-attention matrix, then for regions with a large $\mu_{s,x,y}$, we successively refines the regions to a higher resolution.

\setlength{\textfloatsep}{10pt}
\begin{algorithm}[tb]
   \caption{Constructing the set $\mathcal{J}$}
   \label{alg:algorithm}
   {\small
\begin{algorithmic}
   \STATE {\bfseries Input:} $R = \{s_0, s_1, ..., s_k\}$ in descending order
   \STATE {\bfseries Input:} Budget $m_i$ for each $s_i$ for $i > 0$
   \STATE {\bfseries Input:} Initial $\mathcal{J}$ (empty or prespecified via priors)
   \STATE Compute $\mu_{s_0, x, y}$ for all possible $x, y$ and add $B^{s_0}_{x,y}$ to $\mathcal{J}$
   \FOR{$i = 1$ {\bfseries to} $i = k$}
   \STATE Pop $m_i$ elements $B^{s_{i-1}}_{x,y}$ with the largest $\mu_{s_{i-1}, x, y}$
   \FOR{each $B^{s_{i-1}}_{x,y}$}
   \STATE Compute $\mu_{s_{i}, x', y'}$ for all $\supp(B^{s_i}_{x',y'}) \subseteq \supp(B^{s_{i-1}}_{x,y})$
   \STATE Add $B^{s_i}_{x',y'}$ to $\mathcal{J}$
   \ENDFOR
   \ENDFOR
  \STATE {\bfseries Output:} $\mathcal{J}$
\end{algorithmic}}
\end{algorithm}

With the approximation procedure in place, we can quantify the error of this multiresolution approximation. We only show the approximation error for $R = \{b, 1\}$ for some $b$, but the analysis easily extends beyond $R = \{b, 1\}$. 
\begin{proposition}
\label{prop:error_bound_of_attn}
Let $R = \{b, 1\}$ for some $b$ and $\delta$ be the $m_1$-th largest $\mu_{b, x, y}$, assume for all $(i,j) \in \supp(B^{b}_{x,y})$, $\mathcal{P}_{i,j} \in [a, a + r]$ for some $a$ and $r > 0$, then 
\begin{equation*}
\begin{split}
\frac{||\hat{\mathcal{A}} - \mathcal{A}||_F}{||\mathcal{A}||_F} &\leq \sqrt{\frac{(n^2 - m_1 b^2) C_{2r} \delta^2 }{ \sum_{i,j=1}^n \exp(2 \mathcal{P}_{i,j})}}
\end{split}
\end{equation*}
where $C_{2r} = 1 + \exp(2r) - 2 \exp(r)$
\end{proposition}
Proposition \ref{prop:error_bound_of_attn} again emphasizes the relation between the numerical range $r$ of $\mathcal{P}$ and the quality of an approximation. With some knowledge of the range $r$ and $\mu_{b, x, y}$, we can control the error using an appropriate budget $m_1$. 

\begin{remark}
The procedure shares some commonalities with the correction component of Geometric Multigrid methods \cite{saad_2003, multigrid}. Coarsening is similar to our low resolution approximation, but the prolongation step is different. Rather than interpolate the entire coarse grid to finer grids, our method replaces some regions of coarse grid with its higher resolution approximation. 
\end{remark}

\subsection{How do we compute $\hat{\mathcal{A}} V$?}

We obtained an approximation $\hat{\mathcal{A}}$, but we should not instantiate this matrix (to avoid the $O(n^2)$ cost). So, we discuss a simple procedure for computing $\hat{\mathcal{A}} V$ without constructing the $n \times n$ matrix. 
Define $\tilde{V}_s \in \mathbb{R}^{n/s \times d}$ where $\tilde{V}_1 = V$ and
\begin{equation}
\begin{split}
(\tilde{V}_s)_{i} &= \frac{1}{2} (\tilde{V}_{s/2})_{2i - 1} + \frac{1}{2} (\tilde{V}_{s/2})_{2i}
\end{split}
\label{eq:recursive_pooling_V}
\end{equation}
similar to \eqref{eq:recursive_pooling}. Then, the steps follow Alg. \ref{alg:algorithm_AV}. We again start with multiplying coarse components of $\hat{\mathcal{A}}$ with $V$, then successively add the multiplication of higher resolution components of $\hat{\mathcal{A}}$ and $V$, and finally compute $\hat{\mathcal{A}} V$.

\setlength{\textfloatsep}{10pt}
\begin{algorithm}[tb]
   \caption{Computing $\hat{\mathcal{A}} V$}
   \label{alg:algorithm_AV}
   {\small
\begin{algorithmic}
    \STATE {\bfseries Input:} $R = \{s_0, s_1, ..., s_k\}$ in descending order
   \STATE {\bfseries Input:} $\mathcal{J}, \hat{\mathcal{A}}, \tilde{V}_s$
   \STATE Initialize $Y_{s_{-1}} \in \R^{1 \times d}$ to be zero matrix
   \FOR{$i = 0$ {\bfseries to} $i = k$}
   \STATE Duplicate rows of $Y_{s_{i-1}}$ to create $Y_{s_i} \in \R^{n/s_i \times d}$
   \FOR{each $B^{s_i}_{x,y} \in \mathcal{J}$}
   \STATE $(Y_{s_i})_x = (Y_{s_i})_x + \mu_{s_i,x,y} (\tilde{V}_{s_i})_y$
   \ENDFOR
   \ENDFOR
  \STATE {\bfseries Output:} $Y_{s_k} = \hat{\mathcal{A}} V$
\end{algorithmic}}
\end{algorithm}

\subsection{What is the overall complexity?}
\label{sec:complexity}

We have now described the overall procedure of our approximation approach. In this section, we analyze the complexity of our procedure. Following convention in efficient self-attention papers, we treat $d$ as a constant and it does not influence the complexity. 

We first need to compute $\tilde{Q}_s, \tilde{K}_s, \tilde{V}_s$ for $s \in \{2, 4, \cdots, n\}$. Since each row of $\tilde{Q}_s, \tilde{K}_s, \tilde{V}_s \in \mathbb{R}^{n/s \times d}$ requires averaging over two rows from $\tilde{Q}_{s/2}, \tilde{K}_{s/2}, \tilde{V}_{s/2}$, the total cost of computing $\tilde{Q}_s, \tilde{K}_s, \tilde{V}_s$ for all $s$ is simply $O(\frac{n}{2} + \frac{n}{4} + \cdots + \frac{n}{n}) = O(n)$. 

Given all $\tilde{Q}_s, \tilde{K}_s$, in Alg. \ref{alg:algorithm}, there are $O((n/s_0)^2)$ possible entries of $\mu_{s_0, x, y}$. And at scale $s_i$ for $i > 0$, there are $O(m_i (s_{i-1}/s_{i})^2)$ entries of $\mu_{s_i, x, y}$ since there are $O((s_{i-1}/s_{i})^2)$ number of of $B^{s_i}_{x',y'}$ satisfying $\supp(B^{s_i}_{x',y'}) \subseteq \supp(B^{s_{i-1}}_{x,y})$ and there are $m_i$ regions at scale $s_{i-1}$ to be refined. Note that computing each $\mu_{s, x, y}$ takes $O(1)$, and selecting top-k elements is linear in the input size. Therefore, the cost of constructing $\mathcal{J}$ is $O((n/s_0)^2 + \sum_{i=1}^k m_i (s_{i-1}/s_{i})^2)$. Once $\mathcal{J}$ is constructed, $\hat{\mathcal{A}}$ is simple since $\hat{\mathcal{A}}_{i,j} = \mu_{s,x,y}$ for a $B^{s}_{x,y} \in \mathcal{J}$. 

Finally, multiplying $\hat{A}$ and $V$ in Alg. \ref{alg:algorithm_AV} takes $O(n + (n/s_0)^2 + \sum_{i=1}^k m_i (s_{i-1}/s_{i})^2)$, also. The cost of creating a $Y_{s_i}$ is $O(n/s_i)$, so the cost of creating all $Y_{s}$ for $s \in \{s_0, \cdots, s_k\}$ is $O(\frac{n}{s_0} + \cdots + \frac{n}{s_k}) = O(n)$. Then, for each $B^{s}_{x,y} \in \mathcal{J}$, adding $\mu_{s,x,y} (\hat{V}_{s})_y$ to $(Y_{s})_x$ takes $O(1)$. The size of $\mathcal{J}$ is $O((n/s_0)^2 + \sum_{i=1}^k m_i (s_{i-1}/s_{i})^2)$, so the total complexity of Alg. \ref{alg:algorithm_AV} is as stated. 

Therefore, the total complexity of our approach is $O(n + (n/s_0)^2 + \sum_{i=1}^k m_i (s_{i-1}/s_{i})^2)$. For example, when $R = \{\sqrt{n}, 1\}$, the complexity becomes $O(m_1 n)$. The parameter $m_1$ adjusts the trade-off between approximation accuracy and runtime similar to other efficient methods,  e.g., window size $w$ in $O(wn)$ for Longformer \cite{beltagy2020longformer} and projection size $p$ in $O(pn)$ for Linformer \cite{Wang2020LinformerSW} or Performer \cite{choromanski2020rethinking}. 

\section{Experiments}

We perform a broad set of experiments to evaluate the practical performance profile of our MRA-based self-attention module. {\bf First}, we compare our approximation accuracy with several other baselines. Then, we evaluate 
our method on the RoBERTa language model pretraining \cite{liu2019roberta} and downstream tasks on both short and long sequences. Finally, as is
commonly reported in most evaluations of efficient self-attention methods, we discuss our evaluations on the Long Range Arena (LRA) benchmark \cite{tay2020long}. All hyperparameters are reported 
in \S\ref{sec:appendix_exp}. 

{\bf Overview.} Since the efficiency is a core focus of efficient self-attention methods, time and memory efficiency is taken into account when evaluating performance. Whenever possible, we include runtime and memory consumption of a single instance for each method alongside the accuracy it achieves (in each table). Since the models are exactly the same 
(except which self-attention module is used), we only profile the efficiency of one training step consumed by these modules. See \S\ref{sec:appendix_exp} for more details on profiling. 

{\bf Baselines.} For a rigorous comparison, we use an extensive list of baselines, including Linformer \cite{Wang2020LinformerSW}, Performer \cite{choromanski2020rethinking}, Nystr\"{o}mformer \cite{xiong2021nystromformer}, SOFT \cite{SOFT}, YOSO \cite{yoso}, Reformer \cite{Kitaev2020ReformerTE}, Longformer \cite{beltagy2020longformer}, Big Bird \cite{zaheer2020big}, H-Transformer-1D \cite{htransformer1d}, and Scatterbrain \cite{chen2021scatterbrain}. Since Nystromformer, SOFT, and YOSO also have a variant which involves convolution, we perform evaluations for both cases. We use our multiresolution approximation with $R = \{32, 1\}$ for our method denoted in experiments as MRA-2. Further, we found that in tasks with limited dataset sizes, sparsity provides a regularization 
towards better performance. So, we include a MRA-2-s, which only computes 
\begin{equation}
\hat{\mathcal{A}}_1 = \sum_{B^1_{x, y} \in \mathcal{J}} \mu_{1,x,y} B^1_{x,y}
\end{equation}
after finding $\mathcal{J}$. We use different method-specific hyperparameters for some methods to better understand their efficiency-performance trade off. \textbf{Takeaway}: These detailed comparisons suggest that our MRA-based self-attention offers top performance and top efficiency among the baselines.

\subsection{How good is the approximation accuracy?}

\begin{figure}[!htbp]
\centering
\vspace{-0.05in}
\includegraphics[width=3.25in]{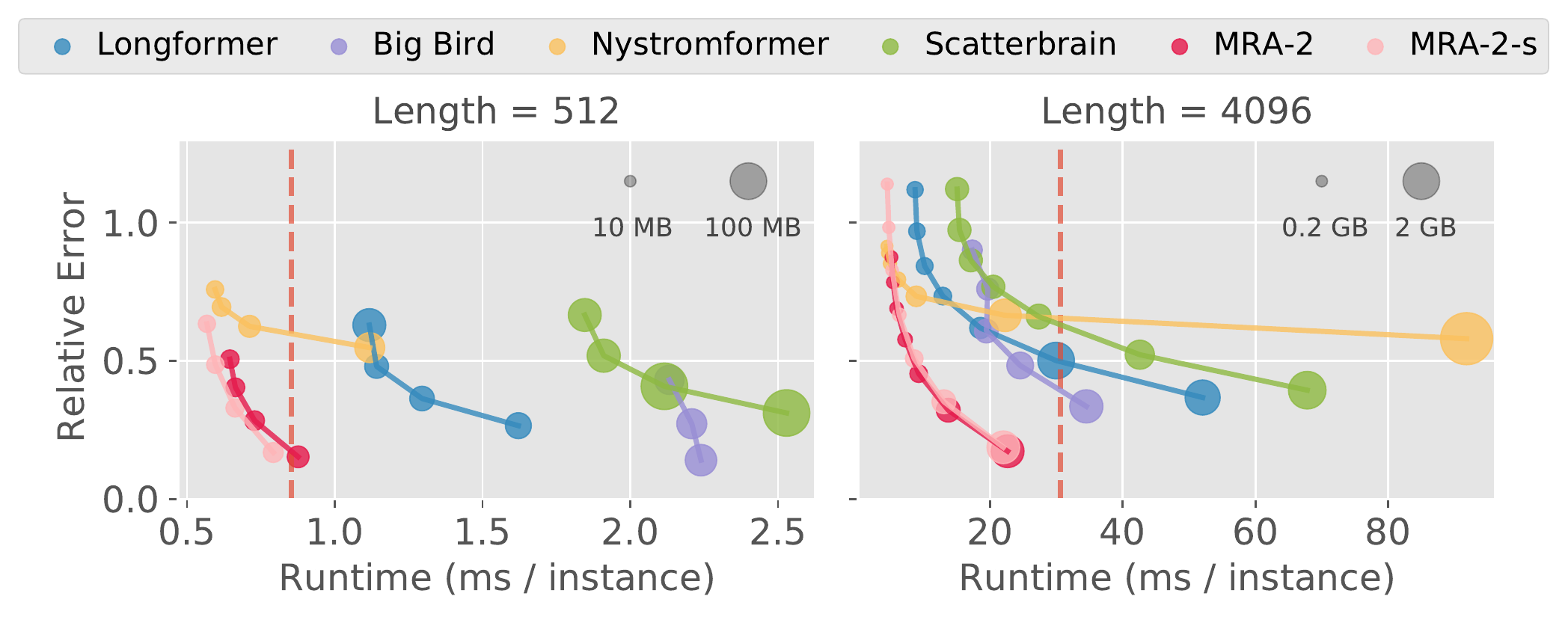}
\vspace{-0.3in}
\caption{Approximation Error vs Runtime vs Memory. 
Red/dotted vertical line is the runtime of standard self-attention. Note that for any points to the right of the red vertical line, the approximation is slower than computing the true self-attention. }
\label{fig:runtime_vs_accuracy}
\vspace{-0.05in}
\end{figure}

We show that our method give the best trade-off between approximation accuracy and efficiency by a significant margin compared to other baselines. 
The approximation accuracy of each method,  compared to the standard self-attention, provides us a direct indication of the performance of approximation methods. To evaluate accuracy, we use $512$ and $4096$ length $Q$, $K$, and $V$ from a pretrained model and compute the relative error
$||\hat{D} \hat{\mathcal{A}} V - D \mathcal{A} V||_F / ||D \mathcal{A} V||_F$. 
As shown in Fig. \ref{fig:runtime_vs_accuracy}, our MRA-2(-s) has the lowest approximation error while maintaining the fastest runtime and smallest memory consumption by a large margin compared to other baselines in both short and long sequences. See \S\ref{sec:appendix_exp} for more details, sequence lengths, and baselines.

\begin{figure}[!bt]
\centering
\includegraphics[width=3.25in]{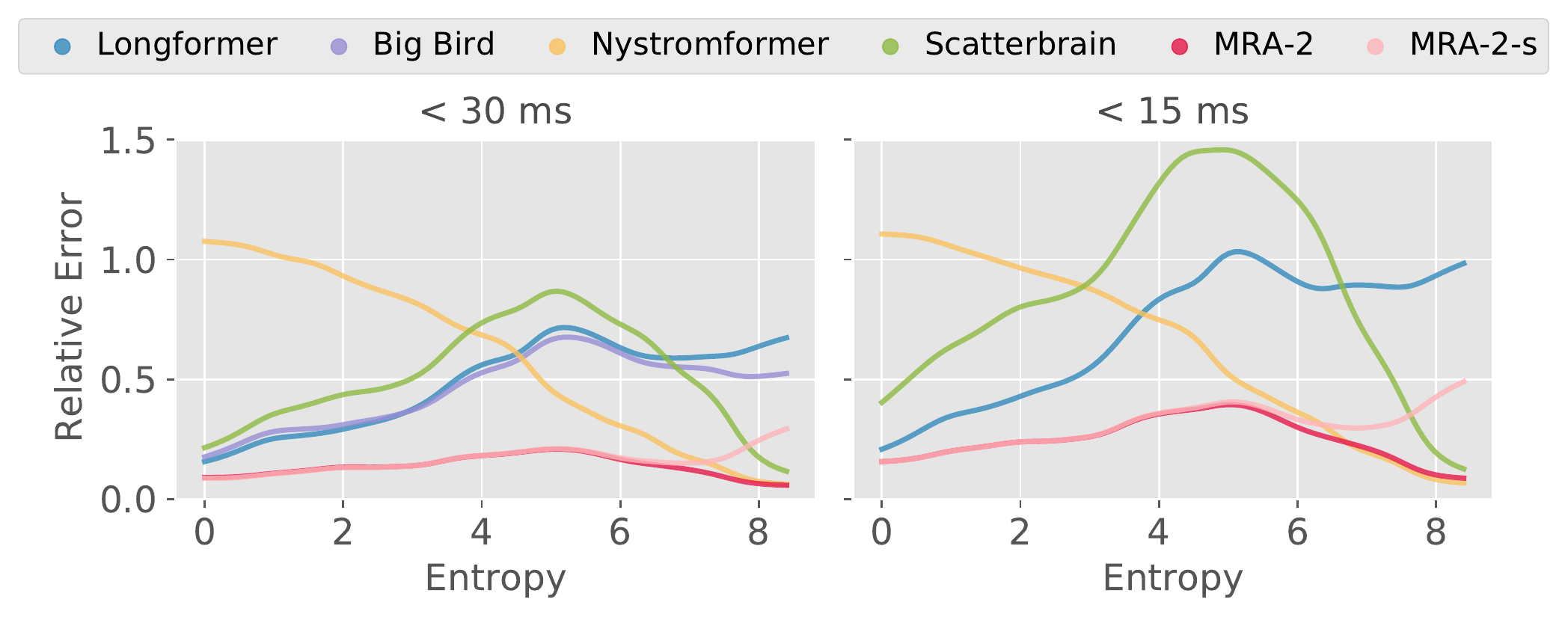}
\vspace{-0.3in}
\caption{Entropy vs Approximation Error plots.  Hyperparameters of each method is set such that the runtime is $<30ms$ and $<15ms$, resp. The fastest setting of Big Bird is still $>15ms$, so omitted from the second plot. Note that the runtime of standard self-attention is roughly 30ms. }
\label{fig:entropy_vs_accuracy}
\vspace{-0.05in}
\end{figure}

Next, we evaluate the effect of the spread (or entropy) of self-attention on the approximation for different methods. The result is shown in Fig. \ref{fig:entropy_vs_accuracy}. We see one limitation of low rank or sparsity-based schemes (discussed in \S\ref{sec:sparse_low_rank_link} and \citet{chen2021scatterbrain}). Our MRA-2 performs well across attention instances with different entropy settings and significantly better than Scatterbrain \cite{chen2021scatterbrain}.

\subsection{RoBERTa Language Modeling}

\begin{table}[!htbp]
\begin{center}
\begin{scriptsize}
\setlength{\tabcolsep}{5pt}
\begin{tabular}{lcccccc}
\toprule
Method & Time & Mem & \multicolumn{2}{c}{MLM} & \multicolumn{2}{c}{MNLI}  \\
 & \tiny{ms} & \tiny{MB} & \tiny{Before} & \tiny{After} & \tiny{m} & \tiny{mm} \\
\midrule
Transformer  &  {\textbf{0.86}}  &  {\transparent{0.3}{71.0}}  &  {\textbf{73.1}}  &  {\textbf{74.0}}  &  {\textbf{87.4}}  &  {\textbf{87.3}} \\
\midrule
Performer  &  {\transparent{0.3}{1.29}}  &  {\transparent{0.3}{62.8}}  &  {\transparent{0.8}{6.8}}  &  {\transparent{0.3}{63.1}}  &  {\transparent{0.3}{32.7}}  &  {\transparent{0.3}{33.0}} \\
\midrule
Linformer  &  {\textbf{0.74}}  &  {\transparent{0.3}{54.5}}  &  {\transparent{0.3}{1.0}}  &  {\transparent{0.3}{5.6}}  &  {\transparent{0.3}{35.4}}  &  {\transparent{0.3}{35.2}} \\
\midrule
SOFT  &  {\textbf{0.86}}  &  {\transparent{0.8}{34.0}}  &  {\transparent{0.8}{10.9}}  &  {\transparent{0.3}{25.0}}  &  {\transparent{0.3}{32.7}}  &  {\transparent{0.3}{33.0}} \\
SOFT + Conv  &  {\transparent{0.8}{1.02}}  &  {\transparent{0.8}{35.5}}  &  {\transparent{0.3}{1.0}}  &  {\transparent{0.3}{65.5}}  &  {\transparent{0.8}{74.9}}  &  {\transparent{0.8}{75.0}} \\
\midrule
Nystromformer  &  {\textbf{0.71}}  &  {\transparent{0.8}{34.8}}  &  {\transparent{0.8}{17.2}}  &  {\transparent{0.8}{68.2}}  &  {\transparent{0.3}{35.4}}  &  {\transparent{0.3}{35.2}} \\
Nystrom + Conv  &  {\transparent{0.8}{0.88}}  &  {\transparent{0.8}{37.2}}  &  {\transparent{0.3}{1.4}}  &  {\transparent{0.8}{70.9}}  &  {\transparent{0.8}{85.1}}  &  {\transparent{0.8}{84.6}} \\
\midrule
YOSO  &  {\transparent{0.8}{0.97}}  &  {\textbf{29.8}}  &  {\transparent{0.8}{13.0}}  &  {\transparent{0.8}{68.4}}  &  {\transparent{0.3}{35.4}}  &  {\transparent{0.3}{35.2}} \\
YOSO + Conv  &  {\transparent{0.3}{1.20}}  &  {\textbf{32.9}}  &  {\transparent{0.3}{3.0}}  &  {\transparent{0.8}{69.0}}  &  {\transparent{0.8}{83.2}}  &  {\transparent{0.8}{83.1}} \\
\midrule
Reformer  &  {\transparent{0.3}{1.23}}  &  {\transparent{0.3}{59.4}}  &  {\transparent{0.3}{0.7}}  &  {\transparent{0.8}{69.5}}  &  {\transparent{0.8}{84.9}}  &  {\transparent{0.8}{85.0}} \\
\midrule
Longformer  &  {\transparent{0.3}{1.30}}  &  {\transparent{0.8}{43.3}}  &  {\transparent{0.8}{66.0}}  &  {\transparent{0.8}{71.2}}  &  {\transparent{0.8}{85.6}}  &  {\transparent{0.8}{85.4}} \\
  &  {\transparent{0.3}{2.31}}  &  {\transparent{0.3}{62.5}}  &  {\textbf{71.9}}  &  {\textbf{73.2}}  &  {\textbf{87.0}}  &  {\textbf{87.1}} \\
\midrule
Big Bird  &  {\transparent{0.3}{2.03}}  &  {\transparent{0.3}{63.9}}  &  {\textbf{71.6}}  &  {\textbf{73.3}}  &  {\textbf{87.1}}  &  {\textbf{87.0}} \\
\midrule
H-Transformer-1D  &  {\transparent{0.8}{0.97}}  &  {\textbf{29.3}}  &  {\transparent{0.3}{0.5}}  &  {\transparent{0.3}{6.1}}  &  {\transparent{0.3}{35.4}}  &  {\transparent{0.3}{35.2}} \\
\midrule
Scatterbrain  &  {\transparent{0.3}{2.23}}  &  {\transparent{0.3}{78.7}}  &  {\transparent{0.8}{60.6}}  &  {\transparent{0.3}{ - }}  &  {\transparent{0.3}{ - }}  &  {\transparent{0.3}{ - }} \\
\midrule
MRA-2  &  {\textbf{0.73}}  &  {\textbf{28.1}}  &  {\textbf{68.9}}  &  {\textbf{73.1}}  &  {\textbf{86.8}}  &  {\textbf{87.1}} \\
  &  {\transparent{0.8}{0.86}}  &  {\transparent{0.8}{34.3}}  &  {\textbf{71.9}}  &  {\textbf{73.8}}  &  {\textbf{87.1}}  &  {\textbf{87.2}} \\
MRA-2-s  &  {\textbf{0.66}}  &  {\textbf{23.8}}  &  {\textbf{67.2}}  &  {\transparent{0.8}{72.8}}  &  {\textbf{87.0}}  &  {\textbf{87.0}} \\
  &  {\textbf{0.80}}  &  {\textbf{29.1}}  &  {\textbf{71.8}}  &  {\textbf{73.8}}  &  {\textbf{87.4}}  &  {\textbf{87.4}} \\
\bottomrule
\end{tabular}
\end{scriptsize}
\end{center}
\vspace{-0.1in}
\caption{Summary of 512 length RoBERTa-base models:  runtime and memory efficiency, MLM accuracy, and MNLI accuracy. Unit for time (and memory) is ms (and MB). Before/After denotes accuracy before/after finetuning. The m/mm give the matched/mismatched MNLI. Some methods have more than one row for different model-specific hyperparameters. We divide measurements into three ranked groups for visualization (bold, normal, transparent). 
}
\label{tab:roberta-base-512}
\vspace{-0.05in}
\end{table}

Here, we use RoBERTa language modeling \cite{liu2019roberta} to assess the performance and efficiency trade off of our method and baselines. 
We use a pretrained RoBERTa-base to evaluate the compatibility of each method with the existing Transformer models and overall feasibility for direct deployment. For fair comparisons, we also check the performance of models trained from scratch. Then, MNLI \cite{williams2018broad} is used to test the model's ability on downstream tasks. 
Further, we extend the 512 length models to 4096 length for a set of best performing methods and use the WikiHop \cite{wikihop} task as an assessment on long sequence language models. 


\textbf{Standard Sequence Length}. Since efficient self-attention approximates standard self-attention, we could simply substitute the standard self-attention of a trained model. This would allow us to minimize the training cost for new methods. To evaluate compatibility with the existing models, we use a pretrained 512 length RoBERTa-base model \cite{liu2019roberta} and replace its self-attention module with efficient alternatives and measure the validation Masked Language Modeling (MLM) accuracy. Then, we check accuracy after finetuning the model on English Wikipedia and Bookcorpus \cite{zhu2015aligning}. Eventually, we finetune the model on the downstream task MNLI \cite{williams2018broad}.

\begin{table}[!htbp]
\begin{center}
\begin{scriptsize}
\setlength{\tabcolsep}{5pt}
\begin{tabular}{llccccc}
\toprule
Method & Time & Mem & MLM & \multicolumn{2}{c}{MNLI}  \\
& \tiny{ms} & \tiny{MB} &  & \tiny{m} & \tiny{mm} \\
\midrule
Transformer  &  {\textbf{0.41}}  &  {\transparent{0.3}{35.47}}  &  {\textbf{57.0}}  &  {\transparent{0.8}{72.7{\fontsize{5}{5}\selectfont $\pm$0.6}}}  &  {\transparent{0.8}{73.8{\fontsize{5}{5}\selectfont $\pm$0.2}}} \\
\midrule
Performer  &  {\transparent{0.3}{0.63}}  &  {\transparent{0.3}{31.38}}  &  {\transparent{0.3}{48.6}}  &  {\transparent{0.3}{69.8{\fontsize{5}{5}\selectfont $\pm$0.4}}}  &  {\transparent{0.3}{70.5{\fontsize{5}{5}\selectfont $\pm$0.1}}} \\
\midrule
Linformer  &  {\textbf{0.35}}  &  {\transparent{0.3}{27.23}}  &  {\transparent{0.8}{53.5}}  &  {\transparent{0.8}{72.5{\fontsize{5}{5}\selectfont $\pm$0.8}}}  &  {\transparent{0.8}{73.2{\fontsize{5}{5}\selectfont $\pm$0.4}}} \\
\midrule
SOFT  &  {\transparent{0.8}{0.43}}  &  {\transparent{0.8}{17.02}}  &  {\transparent{0.3}{42.8}}  &  {\transparent{0.3}{63.8{\fontsize{5}{5}\selectfont $\pm$2.2}}}  &  {\transparent{0.3}{64.7{\fontsize{5}{5}\selectfont $\pm$2.6}}} \\
SOFT + Conv  &  {\transparent{0.8}{0.53}}  &  {\transparent{0.8}{17.77}}  &  {\transparent{0.8}{56.7}}  &  {\transparent{0.3}{70.8{\fontsize{5}{5}\selectfont $\pm$0.5}}}  &  {\transparent{0.3}{71.8{\fontsize{5}{5}\selectfont $\pm$0.4}}} \\
\midrule
Nystromformer  &  {\textbf{0.34}}  &  {\transparent{0.8}{17.40}}  &  {\transparent{0.3}{53.1}}  &  {\transparent{0.3}{71.4{\fontsize{5}{5}\selectfont $\pm$0.6}}}  &  {\transparent{0.3}{72.0{\fontsize{5}{5}\selectfont $\pm$0.3}}} \\
Nystrom + Conv  &  {\transparent{0.8}{0.45}}  &  {\transparent{0.8}{18.60}}  &  {\textbf{57.3}}  &  {\transparent{0.8}{73.0{\fontsize{5}{5}\selectfont $\pm$0.4}}}  &  {\transparent{0.8}{73.9{\fontsize{5}{5}\selectfont $\pm$0.6}}} \\
\midrule
YOSO  &  {\transparent{0.8}{0.47}}  &  {\textbf{14.91}}  &  {\transparent{0.3}{53.4}}  &  {\transparent{0.8}{72.9{\fontsize{5}{5}\selectfont $\pm$0.8}}}  &  {\transparent{0.8}{73.2{\fontsize{5}{5}\selectfont $\pm$0.4}}} \\
YOSO + Conv  &  {\transparent{0.3}{0.58}}  &  {\textbf{16.42}}  &  {\textbf{57.2}}  &  {\transparent{0.8}{72.5{\fontsize{5}{5}\selectfont $\pm$0.4}}}  &  {\transparent{0.3}{72.9{\fontsize{5}{5}\selectfont $\pm$0.5}}} \\
\midrule
Reformer  &  {\textbf{0.39}}  &  {\textbf{16.43}}  &  {\transparent{0.3}{52.4}}  &  {\textbf{73.7{\fontsize{5}{5}\selectfont $\pm$0.4}}}  &  {\textbf{74.6{\fontsize{5}{5}\selectfont $\pm$0.3}}} \\
  &  {\transparent{0.3}{0.61}}  &  {\transparent{0.3}{29.65}}  &  {\transparent{0.8}{55.6}}  &  {\textbf{75.0{\fontsize{5}{5}\selectfont $\pm$0.2}}}  &  {\textbf{75.6{\fontsize{5}{5}\selectfont $\pm$0.3}}} \\
\midrule
Longformer  &  {\transparent{0.3}{0.61}}  &  {\transparent{0.8}{21.60}}  &  {\transparent{0.8}{54.7}}  &  {\transparent{0.3}{72.0{\fontsize{5}{5}\selectfont $\pm$0.4}}}  &  {\transparent{0.8}{73.5{\fontsize{5}{5}\selectfont $\pm$0.2}}} \\
  &  {\transparent{0.3}{1.10}}  &  {\transparent{0.3}{31.44}}  &  {\textbf{57.4}}  &  {\textbf{75.8{\fontsize{5}{5}\selectfont $\pm$0.5}}}  &  {\textbf{76.7{\fontsize{5}{5}\selectfont $\pm$0.6}}} \\
\midrule
Big Bird  &  {\transparent{0.3}{1.02}}  &  {\transparent{0.3}{31.91}}  &  {\textbf{57.6}}  &  {\textbf{75.0{\fontsize{5}{5}\selectfont $\pm$0.5}}}  &  {\textbf{75.6{\fontsize{5}{5}\selectfont $\pm$0.6}}} \\
\midrule
H-Transformer-1D  &  {\transparent{0.8}{0.47}}  &  {\textbf{14.65}}  &  {\transparent{0.3}{43.7}}  &  {\transparent{0.3}{62.9{\fontsize{5}{5}\selectfont $\pm$2.7}}}  &  {\transparent{0.3}{63.4{\fontsize{5}{5}\selectfont $\pm$3.9}}} \\
\midrule
Scatterbrain  &  {\transparent{0.3}{1.04}}  &  {\transparent{0.3}{78.66}}  &  {\transparent{0.3}{20.5}}  &  {\transparent{0.3}{42.6{\fontsize{5}{5}\selectfont $\pm$8.1}}}  &  {\transparent{0.3}{43.4{\fontsize{5}{5}\selectfont $\pm$9.5}}} \\
\midrule
MRA-2  &  {\textbf{0.36}}  &  {\textbf{14.05}}  &  {\transparent{0.8}{56.4}}  &  {\textbf{73.2{\fontsize{5}{5}\selectfont $\pm$0.2}}}  &  {\textbf{74.1{\fontsize{5}{5}\selectfont $\pm$0.5}}} \\
  &  {\transparent{0.8}{0.43}}  &  {\transparent{0.8}{17.15}}  &  {\textbf{57.3}}  &  {\transparent{0.8}{73.0{\fontsize{5}{5}\selectfont $\pm$1.0}}}  &  {\transparent{0.8}{73.9{\fontsize{5}{5}\selectfont $\pm$0.8}}} \\
MRA-2-s  &  {\textbf{0.31}}  &  {\textbf{11.93}}  &  {\transparent{0.8}{56.7}}  &  {\textbf{73.6{\fontsize{5}{5}\selectfont $\pm$1.6}}}  &  {\textbf{74.3{\fontsize{5}{5}\selectfont $\pm$1.1}}} \\
  &  {\textbf{0.38}}  &  {\textbf{14.57}}  &  {\textbf{57.5}}  &  {\textbf{73.9{\fontsize{5}{5}\selectfont $\pm$0.6}}}  &  {\textbf{74.6{\fontsize{5}{5}\selectfont $\pm$0.8}}} \\
\bottomrule
\end{tabular}
\end{scriptsize}
\end{center}
\vspace{-0.1in}
\caption{Summary of 512 length RoBERTa-small models. We also include a 95\% error bar for experiments that have a small compute burden. 
}
\label{tab:roberta-small-512}
\vspace{-0.1in}
\end{table}

Only a handful of schemes including Longformer, Big Bird, and MRA-2(-s) are fully compatible with pretrained models. Scatterbrain has a reasonable accuracy without further finetuning, but the training diverges when finetuning the model. The  other methods cannot get a satisfactory level of accuracy. These statements also hold for the downstream finetuning results, shown in Tab. \ref{tab:roberta-base-512}. Our method has the best performance among baselines for both MLM and MNLI. 
Meanwhile, it has a much better time and memory efficiency.

\begin{table}[!htbp]
\vspace{-0.03in}
\begin{center}
\begin{scriptsize}
\setlength{\tabcolsep}{5pt}
\begin{tabular}{lcccc}
\toprule
Method & Time (ms) & Mem (GB) & MLM & WikiHop  \\
\midrule
Transformer  &  {\transparent{0.3}{30.88}}  &  {\transparent{0.3}{3.93}}  &  {\textbf{74.3}}  &  {\textbf{74.6}} \\
\midrule
Longformer  &  {\transparent{0.8}{10.20}}  &  {\transparent{0.8}{0.35}}  &  {\transparent{0.3}{71.1}}  &  {\transparent{0.3}{60.8}} \\
\midrule
Big Bird  &  {\transparent{0.3}{17.53}}  &  {\transparent{0.3}{0.59}}  &  {\transparent{0.3}{-}}  &  {\transparent{0.3}{-}} \\
\midrule
MRA-2  &  {\textbf{7.03}}  &  {\textbf{0.28}}  &  {\transparent{0.8}{73.1}}  &  {\transparent{0.8}{71.2}} \\
 &  {\transparent{0.8}{9.25}}  &  {\transparent{0.8}{0.38}}  &  {\textbf{73.7}}  &  {\textbf{73.4}} \\
MRA-2-s  &  {\textbf{6.37}}  &  {\textbf{0.23}}  &  {\transparent{0.8}{73.0}}  &  {\transparent{0.8}{71.8}} \\
  &  {\textbf{8.62}}  &  {\transparent{0.8}{0.38}}  &  {\textbf{73.8}}  &  {\textbf{74.1}} \\
\bottomrule
\end{tabular}
\end{scriptsize}
\end{center}
\vspace{-0.12in}
\caption{Summary of 4096 length RoBERTa-base models. Since Big Bird is slow and we are not able to reduce its training time using multiple GPU, we cannot can test Big Bird for 4096 sequence. 
}
\label{tab:roberta-base-4096}
\vspace{-0.1in}
\end{table}

Since many baselines are not compatible with the trained model weights (performance degrades when substituting the self-attention module), to make the comparison fair for all methods, we also evaluate models trained from scratch. Due to the large number of baselines we use, we train a small variant of RoBERTa on English Wikipedia and BookCorpus \cite{zhu2015aligning} to keep the training cost reasonable. Then, we again finetune the model on downstream task (MNLI). Results are summarized in Tab. \ref{tab:roberta-small-512}. Only a few methods (including ours) achieve both good performance and efficiency.

\begin{table}[!htbp]
\vspace{-0.03in}
\begin{center}
\begin{scriptsize}
\setlength{\tabcolsep}{5pt}
\begin{tabular}{lcccc}
\toprule
Method & Time (ms) & Mem (GB) & MLM & WikiHop  \\
\midrule
Transformer  &  {\transparent{0.3}{15.36}}  &  {\transparent{0.3}{1.96}}  &  {\textbf{55.8}}  &  {\textbf{54.6{\fontsize{5}{5}\selectfont $\pm$1.6}}} \\
\midrule
Performer  &  {\transparent{0.3}{5.13}}  &  {\transparent{0.3}{0.24}}  &  {\transparent{0.3}{23.2}}  &  {\transparent{0.3}{43.7{\fontsize{5}{5}\selectfont $\pm$0.6}}} \\
\midrule
Linformer  &  {\textbf{2.85}}  &  {\transparent{0.8}{0.21}}  &  {\transparent{0.3}{13.8}}  &  {\transparent{0.3}{11.0{\fontsize{5}{5}\selectfont $\pm$0.4}}} \\
\midrule
SOFT  &  {\textbf{2.46}}  &  {\textbf{0.11}}  &  {\transparent{0.3}{25.9}}  &  {\transparent{0.3}{14.0{\fontsize{5}{5}\selectfont $\pm$8.6}}} \\
  &  {\transparent{0.3}{5.92}}  &  {\transparent{0.3}{0.24}}  &  {\transparent{0.3}{31.0}}  &  {\transparent{0.3}{12.1{\fontsize{5}{5}\selectfont $\pm$1.9}}} \\
SOFT + Conv  &  {\textbf{3.33}}  &  {\textbf{0.11}}  &  {\transparent{0.8}{52.8}}  &  {\transparent{0.3}{30.8{\fontsize{5}{5}\selectfont $\pm$29.3}}} \\
\midrule
Nystromformer  &  {\textbf{2.38}}  &  {\textbf{0.11}}  &  {\transparent{0.3}{34.7}}  &  {\transparent{0.8}{44.0{\fontsize{5}{5}\selectfont $\pm$0.2}}} \\
  &  {\transparent{0.8}{4.34}}  &  {\transparent{0.3}{0.27}}  &  {\transparent{0.8}{46.8}}  &  {\transparent{0.8}{46.0{\fontsize{5}{5}\selectfont $\pm$0.8}}} \\
Nystrom + Conv  &  {\textbf{3.23}}  &  {\textbf{0.12}}  &  {\transparent{0.8}{53.1}}  &  {\textbf{54.6{\fontsize{5}{5}\selectfont $\pm$0.8}}} \\
\midrule
YOSO  &  {\transparent{0.8}{4.15}}  &  {\textbf{0.12}}  &  {\transparent{0.8}{47.8}}  &  {\transparent{0.8}{52.4{\fontsize{5}{5}\selectfont $\pm$0.1}}} \\
  &  {\transparent{0.3}{5.07}}  &  {\transparent{0.8}{0.17}}  &  {\transparent{0.8}{49.9}}  &  {\transparent{0.8}{52.8{\fontsize{5}{5}\selectfont $\pm$0.5}}} \\
YOSO + Conv  &  {\transparent{0.3}{5.45}}  &  {\transparent{0.8}{0.13}}  &  {\textbf{55.1}}  &  {\textbf{53.2{\fontsize{5}{5}\selectfont $\pm$0.7}}} \\
\midrule
Reformer  &  {\transparent{0.8}{5.04}}  &  {\transparent{0.3}{0.24}}  &  {\transparent{0.8}{52.2}}  &  {\textbf{53.7{\fontsize{5}{5}\selectfont $\pm$0.9}}} \\
\midrule
Longformer  &  {\transparent{0.8}{4.88}}  &  {\transparent{0.8}{0.17}}  &  {\transparent{0.8}{52.4}}  &  {\transparent{0.8}{52.3{\fontsize{5}{5}\selectfont $\pm$0.7}}} \\
\midrule
Big Bird  &  {\transparent{0.3}{8.68}}  &  {\transparent{0.3}{0.29}}  &  {\textbf{54.4}}  &  {\textbf{54.3{\fontsize{5}{5}\selectfont $\pm$0.7}}} \\
\midrule
H-Transformer-1D  &  {\transparent{0.8}{3.93}}  &  {\textbf{0.12}}  &  {\transparent{0.3}{41.1}}  &  {\transparent{0.3}{43.7{\fontsize{5}{5}\selectfont $\pm$0.7}}} \\
\midrule
Scatterbrain  &  {\transparent{0.3}{8.83}}  &  {\transparent{0.3}{0.31}}  &  {\transparent{0.3}{35.8}}  &  {\transparent{0.3}{12.1{\fontsize{5}{5}\selectfont $\pm$0.9}}} \\
\midrule
MRA-2  &  {\textbf{3.43}}  &  {\transparent{0.8}{0.14}}  &  {\textbf{54.2}}  &  {\transparent{0.8}{52.6{\fontsize{5}{5}\selectfont $\pm$0.9}}} \\
  &  {\transparent{0.8}{4.52}}  &  {\transparent{0.8}{0.19}}  &  {\textbf{55.2}}  &  {\textbf{54.0{\fontsize{5}{5}\selectfont $\pm$0.9}}} \\
MRA-2-s  &  {\textbf{3.12}}  &  {\textbf{0.12}}  &  {\textbf{53.8}}  &  {\transparent{0.8}{51.8{\fontsize{5}{5}\selectfont $\pm$0.9}}} \\
  &  {\transparent{0.8}{4.13}}  &  {\transparent{0.8}{0.19}}  &  {\textbf{55.1}}  &  {\textbf{53.6{\fontsize{5}{5}\selectfont $\pm$0.8}}} \\
\bottomrule
\end{tabular}
\end{scriptsize}
\end{center}
\vspace{-0.12in}
\caption{Summary of 4096 length RoBERTa-small models. 
}
\label{tab:roberta-small-4096}
\vspace{-0.14in}
\end{table}

\textbf{Longer Sequences}. To evaluate the performance of our MRA-2(-s) on longer sequences, we extend the 512 length models to 4096 length. We extend the positional embedding and further train the models on English Wikipedia, Bookcorpus \cite{zhu2015aligning}, one third of Stories dataset \cite{trinh2019simple}, and one third of RealNews dataset \cite{zellers2019defending} following \cite{beltagy2020longformer}. Then, the 4096 length models are finetuned on WikiHop dataset \cite{wikihop} to assess the performance of these models on downstream tasks. The results are summarized in Tab. \ref{tab:roberta-base-4096} for base models and Tab. \ref{tab:roberta-small-4096} for small models. Our MRA-2 is again one of the top performing methods with high efficiency among baselines. Note that the difference in WikiHop performance of Longformer \cite{beltagy2020longformer} from the original paper is due to a much larger window size which has an even slower runtime. Linformer \cite{Wang2020LinformerSW} does not seem to be able to adapt the weights from its 512 length model to a 4096 model. It is interesting that the convolution in Nystromf\"{o}rmer \cite{xiong2021nystromformer} seems to play an important role in boosting performance.

\subsection{Long Range Arena}

\begin{table}[!htbp]
\vspace{-0.03in}
\begin{center}
\begin{scriptsize}
\setlength{\tabcolsep}{2pt}
\begin{tabular}{lcccccc}
\toprule
Method & Listops & Text & Retrieval & Image & Pathfinder & Avg  \\
\midrule
Transformer  &  {\transparent{0.8}{37.1{\fontsize{5}{5}\selectfont $\pm$0.4}}}  &  {\textbf{65.2{\fontsize{5}{5}\selectfont $\pm$0.6}}}  &  {\transparent{0.8}{79.6{\fontsize{5}{5}\selectfont $\pm$1.7}}}  &  {\transparent{0.3}{38.5{\fontsize{5}{5}\selectfont $\pm$0.7}}}  &  {\textbf{72.8{\fontsize{5}{5}\selectfont $\pm$1.1}}}  &  {\transparent{0.8}{58.7{\fontsize{5}{5}\selectfont $\pm$0.3}}} \\
\midrule
Performer  &  {\transparent{0.3}{36.7{\fontsize{5}{5}\selectfont $\pm$0.2}}}  &  {\transparent{0.8}{65.2{\fontsize{5}{5}\selectfont $\pm$0.9}}}  &  {\transparent{0.3}{79.5{\fontsize{5}{5}\selectfont $\pm$1.4}}}  &  {\transparent{0.3}{38.6{\fontsize{5}{5}\selectfont $\pm$0.7}}}  &  {\transparent{0.8}{71.4{\fontsize{5}{5}\selectfont $\pm$0.7}}}  &  {\transparent{0.8}{58.3{\fontsize{5}{5}\selectfont $\pm$0.1}}} \\
\midrule
Linformer  &  {\textbf{37.4{\fontsize{5}{5}\selectfont $\pm$0.3}}}  &  {\transparent{0.3}{57.0{\fontsize{5}{5}\selectfont $\pm$1.1}}}  &  {\transparent{0.3}{78.4{\fontsize{5}{5}\selectfont $\pm$0.1}}}  &  {\transparent{0.3}{38.1{\fontsize{5}{5}\selectfont $\pm$0.3}}}  &  {\transparent{0.3}{67.2{\fontsize{5}{5}\selectfont $\pm$0.1}}}  &  {\transparent{0.3}{55.6{\fontsize{5}{5}\selectfont $\pm$0.3}}} \\
\midrule
SOFT  &  {\transparent{0.3}{36.3{\fontsize{5}{5}\selectfont $\pm$1.4}}}  &  {\transparent{0.8}{65.2{\fontsize{5}{5}\selectfont $\pm$0.0}}}  &  {\textbf{83.3{\fontsize{5}{5}\selectfont $\pm$1.0}}}  &  {\transparent{0.3}{35.3{\fontsize{5}{5}\selectfont $\pm$1.3}}}  &  {\transparent{0.3}{67.7{\fontsize{5}{5}\selectfont $\pm$1.1}}}  &  {\transparent{0.3}{57.5{\fontsize{5}{5}\selectfont $\pm$0.5}}} \\
SOFT + Conv  &  {\transparent{0.8}{37.1{\fontsize{5}{5}\selectfont $\pm$0.4}}}  &  {\transparent{0.8}{65.2{\fontsize{5}{5}\selectfont $\pm$0.4}}}  &  {\textbf{82.9{\fontsize{5}{5}\selectfont $\pm$0.0}}}  &  {\transparent{0.3}{37.1{\fontsize{5}{5}\selectfont $\pm$4.7}}}  &  {\transparent{0.3}{68.1{\fontsize{5}{5}\selectfont $\pm$0.4}}}  &  {\transparent{0.8}{58.1{\fontsize{5}{5}\selectfont $\pm$0.9}}} \\
\midrule
Nystromformer  &  {\transparent{0.3}{24.7{\fontsize{5}{5}\selectfont $\pm$17.5}}}  &  {\textbf{65.7{\fontsize{5}{5}\selectfont $\pm$0.1}}}  &  {\textbf{80.2{\fontsize{5}{5}\selectfont $\pm$0.3}}}  &  {\transparent{0.8}{38.8{\fontsize{5}{5}\selectfont $\pm$2.9}}}  &  {\textbf{73.1{\fontsize{5}{5}\selectfont $\pm$0.1}}}  &  {\transparent{0.3}{56.5{\fontsize{5}{5}\selectfont $\pm$2.8}}} \\
Nystrom + Conv  &  {\transparent{0.3}{30.6{\fontsize{5}{5}\selectfont $\pm$8.9}}}  &  {\textbf{65.7{\fontsize{5}{5}\selectfont $\pm$0.2}}}  &  {\transparent{0.3}{78.9{\fontsize{5}{5}\selectfont $\pm$1.2}}}  &  {\textbf{43.2{\fontsize{5}{5}\selectfont $\pm$3.4}}}  &  {\transparent{0.3}{69.1{\fontsize{5}{5}\selectfont $\pm$1.0}}}  &  {\transparent{0.3}{57.5{\fontsize{5}{5}\selectfont $\pm$1.5}}} \\
\midrule
YOSO  &  {\transparent{0.8}{37.0{\fontsize{5}{5}\selectfont $\pm$0.3}}}  &  {\transparent{0.3}{63.1{\fontsize{5}{5}\selectfont $\pm$0.2}}}  &  {\transparent{0.3}{78.3{\fontsize{5}{5}\selectfont $\pm$0.7}}}  &  {\transparent{0.8}{40.8{\fontsize{5}{5}\selectfont $\pm$0.8}}}  &  {\textbf{72.9{\fontsize{5}{5}\selectfont $\pm$0.6}}}  &  {\transparent{0.8}{58.4{\fontsize{5}{5}\selectfont $\pm$0.3}}} \\
YOSO + Conv  &  {\textbf{37.2{\fontsize{5}{5}\selectfont $\pm$0.5}}}  &  {\transparent{0.8}{64.9{\fontsize{5}{5}\selectfont $\pm$1.2}}}  &  {\transparent{0.3}{78.5{\fontsize{5}{5}\selectfont $\pm$0.9}}}  &  {\textbf{44.6{\fontsize{5}{5}\selectfont $\pm$0.7}}}  &  {\transparent{0.8}{69.5{\fontsize{5}{5}\selectfont $\pm$3.5}}}  &  {\textbf{59.0{\fontsize{5}{5}\selectfont $\pm$1.1}}} \\
\midrule
Reformer  &  {\transparent{0.3}{18.9{\fontsize{5}{5}\selectfont $\pm$2.4}}}  &  {\transparent{0.8}{64.9{\fontsize{5}{5}\selectfont $\pm$0.4}}}  &  {\transparent{0.3}{78.2{\fontsize{5}{5}\selectfont $\pm$1.6}}}  &  {\textbf{42.4{\fontsize{5}{5}\selectfont $\pm$0.4}}}  &  {\transparent{0.3}{68.9{\fontsize{5}{5}\selectfont $\pm$1.1}}}  &  {\transparent{0.3}{54.7{\fontsize{5}{5}\selectfont $\pm$0.2}}} \\
\midrule
Longformer  &  {\transparent{0.8}{37.2{\fontsize{5}{5}\selectfont $\pm$0.3}}}  &  {\transparent{0.3}{64.1{\fontsize{5}{5}\selectfont $\pm$0.1}}}  &  {\transparent{0.8}{79.7{\fontsize{5}{5}\selectfont $\pm$1.1}}}  &  {\textbf{42.6{\fontsize{5}{5}\selectfont $\pm$0.1}}}  &  {\transparent{0.8}{70.7{\fontsize{5}{5}\selectfont $\pm$0.8}}}  &  {\textbf{58.9{\fontsize{5}{5}\selectfont $\pm$0.1}}} \\
\midrule
Big Bird  &  {\textbf{37.4{\fontsize{5}{5}\selectfont $\pm$0.3}}}  &  {\transparent{0.3}{64.3{\fontsize{5}{5}\selectfont $\pm$1.1}}}  &  {\transparent{0.8}{79.9{\fontsize{5}{5}\selectfont $\pm$0.1}}}  &  {\transparent{0.8}{40.9{\fontsize{5}{5}\selectfont $\pm$1.1}}}  &  {\textbf{72.6{\fontsize{5}{5}\selectfont $\pm$0.7}}}  &  {\textbf{59.0{\fontsize{5}{5}\selectfont $\pm$0.3}}} \\
\midrule
H-Transformer-1D  &  {\transparent{0.3}{30.4{\fontsize{5}{5}\selectfont $\pm$8.8}}}  &  {\textbf{66.0{\fontsize{5}{5}\selectfont $\pm$0.2}}}  &  {\textbf{80.1{\fontsize{5}{5}\selectfont $\pm$0.4}}}  &  {\textbf{42.1{\fontsize{5}{5}\selectfont $\pm$0.8}}}  &  {\transparent{0.8}{70.7{\fontsize{5}{5}\selectfont $\pm$0.1}}}  &  {\transparent{0.8}{57.8{\fontsize{5}{5}\selectfont $\pm$1.8}}} \\
\midrule
Scatterbrain  &  {\textbf{37.5{\fontsize{5}{5}\selectfont $\pm$0.1}}}  &  {\transparent{0.3}{64.4{\fontsize{5}{5}\selectfont $\pm$0.3}}}  &  {\transparent{0.8}{79.6{\fontsize{5}{5}\selectfont $\pm$0.1}}}  &  {\transparent{0.3}{38.0{\fontsize{5}{5}\selectfont $\pm$0.9}}}  &  {\transparent{0.3}{54.8{\fontsize{5}{5}\selectfont $\pm$7.8}}}  &  {\transparent{0.3}{54.9{\fontsize{5}{5}\selectfont $\pm$1.4}}} \\
\midrule
MRA-2  &  {\transparent{0.8}{37.2{\fontsize{5}{5}\selectfont $\pm$0.3}}}  &  {\textbf{65.4{\fontsize{5}{5}\selectfont $\pm$0.1}}}  &  {\transparent{0.8}{79.6{\fontsize{5}{5}\selectfont $\pm$0.6}}}  &  {\transparent{0.8}{39.5{\fontsize{5}{5}\selectfont $\pm$0.9}}}  &  {\textbf{73.6{\fontsize{5}{5}\selectfont $\pm$0.4}}}  &  {\textbf{59.0{\fontsize{5}{5}\selectfont $\pm$0.3}}} \\
MRA-2-s  &  {\textbf{37.4{\fontsize{5}{5}\selectfont $\pm$0.5}}}  &  {\transparent{0.3}{64.3{\fontsize{5}{5}\selectfont $\pm$0.8}}}  &  {\textbf{80.3{\fontsize{5}{5}\selectfont $\pm$0.1}}}  &  {\transparent{0.8}{41.1{\fontsize{5}{5}\selectfont $\pm$0.4}}}  &  {\textbf{73.8{\fontsize{5}{5}\selectfont $\pm$0.6}}}  &  {\textbf{59.4{\fontsize{5}{5}\selectfont $\pm$0.2}}} \\
\bottomrule
\end{tabular}
\end{scriptsize}
\end{center}
\vspace{-0.12in}
\caption{Test set accuracy of LRA tasks. Since the benchmark consist of multiple tasks with different sequence length, we do not include the efficiency components in the table. 
}
\label{tab:lra}
\vspace{-0.05in}
\end{table}

The Long Range Arena (LRA) \cite{tay2020long} has been proposed to provide a lightweight benchmark to quickly compare the capability of long sequence modeling for Transformers. 
Due to a consistency issue and code compatibility of official LRA benchmark (see \href{https://github.com/google-research/long-range-arena/issues/34}{Issue-34}, \href{https://github.com/google-research/long-range-arena/issues/35}{Issue-35}, and \citet{lee2021fnet}), we use the LRA code provided by \cite{xiong2021nystromformer} and follow exactly the same hyperparameter setting. The results are shown in Tab. \ref{tab:lra}. Our method has the best performance compared to others. 

{\bf Caveats.} A reader may ask why Longformer, Big Bird, and MRA-2-s perform better than standard Transformers \cite{vaswani2017attention} despite being approximations.  The performance difference is most obvious on the image task. We also found that Longformer with a smaller local attention window ($256$, $128$, $64$) tends to offer better performance ($39.52$, $42.11$, $42.71$, respectively) on the image task. One reason is that standard self-attention needs larger datasets to compensate for its lack of locality bias \cite{xu2021vitae, dascoli2021convit}. Hence, due to the small datasets (i.e., its lightweight nature), the LRA accuracy metrics should be interpreted with caution. 

\begin{table}[!htbp]
\vspace{-0.03in}
\begin{center}
\begin{scriptsize}
\setlength{\tabcolsep}{5pt}
\begin{tabular}{lcccc}
\toprule
Method & Time (ms) & Mem (MB) & Top-1 & Top-5  \\
\midrule
Transformer  &  {\transparent{0.3}{1.24}}  &  {\transparent{0.3}{45.5}}  &  {\transparent{0.8}{48.7}}  &  {\transparent{0.8}{73.7}} \\
\midrule
Reformer  &  {\transparent{0.3}{1.14}}  &  {\transparent{0.3}{19.1}}  &  {\transparent{0.3}{39.6}}  &  {\transparent{0.3}{65.5}} \\
\midrule
Longformer  &  {\transparent{0.8}{1.12}}  &  {\transparent{0.8}{13.7}}  &  {\textbf{49.1}}  &  {\transparent{0.8}{73.9}} \\
\midrule
H-Transformer-1D  &  {\transparent{0.8}{1.03}}  &  {\textbf{9.8}}  &  {\transparent{0.3}{48.7}}  &  {\textbf{73.9}} \\
\midrule
MRA-2  &  {\textbf{1.00}}  &  {\transparent{0.8}{11.8}}  &  {\transparent{0.8}{48.9}}  &  {\transparent{0.3}{73.6}} \\
MRA-2-s  &  {\textbf{0.98}}  &  {\textbf{9.7}}  &  {\textbf{49.2}}  &  {\textbf{73.9}} \\
\bottomrule
\end{tabular}
\end{scriptsize}
\end{center}
\vspace{-0.12in}
\caption{Summary of ImageNet results trained on 4-layer Transformers. We reports both top-1 and top-5 accuracy. 
}
\label{tab:imagenet}
\vspace{-0.12in}
\end{table}

\textbf{ImageNet}. To test the performance on large datasets, we use ImageNet \cite{imagenet} as a large scale alternative to CIFAR-10 \citep{krizhevsky2009learning} used in image task of LRA. Further, data augmentation is used to increase the dataset size. Like LRA, we focus on small models and use a 4-layer Transformer (see \S\ref{sec:appendix_exp} for more details). 
Model specific hyperparameters are the same as the ones used on LRA. The results are shown in Tab. \ref{tab:imagenet}. MRA-2-s is the top performing approach. Standard self-attention and MRA-2 can clearly perform better on a large dataset.

\section{Conclusion}

We show that Multiresolution analysis (MRA) 
provides fresh 
ideas for efficiently approximating 
self-attention, which subsumes many piecemeal 
approaches in the literature. 
We expect that 
exploiting the links to MRA will allow
leveraging a vast body of technical results developed 
over many decades. But we show that there 
are tangible practical benefits available immediately. 
When some consideration is given to which design choices or heuristics for a MRA-based self-attention scheme will interface well with mature software stacks 
and modern hardware, we obtain 
a procedure with strong advantages 
across both {\em performance/accuracy and efficiency}. 
Further, our implementation can be directly plugged into existing Transformers, a feature missing in some existing 
efficient transformer implementations. We show use cases 
on longer sequence tasks and in resource limited setting 
but believe that various other applications of Transformers will also benefit in the short term.
Finally, we should note the lack of integrated software support for MRA as well as our specialized model in current deep learning libraries. Overcoming this limitation required implementing custom CUDA kernels for some generic block sparsity operators. Therefore, extending our algorithm for other use cases may involve reimplementing the kernel. We hope that with broader use of MRA-based methods, the software support will improve thereby reducing this implementation barrier.

\section*{Acknowledgments}

This work was supported by the UW AmFam Data Science Institute through funds from American Family Insurance. VS was supported by NIH RF1 AG059312. We thank Sathya Ravi for  discussions regarding multigrid methods, Karu Sankaralingam for suggestions regarding hardware support for sparsity/block sparsity, and Pranav Pulijala for integrating our algorithm within the HuggingFace library.


\bibliography{main}

\begin{thebibliography}{43}
\providecommand{\natexlab}[1]{#1}
\providecommand{\url}[1]{\texttt{#1}}
\expandafter\ifx\csname urlstyle\endcsname\relax
  \providecommand{\doi}[1]{doi: #1}\else
  \providecommand{\doi}{doi: \begingroup \urlstyle{rm}\Url}\fi

\bibitem[Beltagy et~al.(2020)Beltagy, Peters, and Cohan]{beltagy2020longformer}
Beltagy, I., Peters, M.~E., and Cohan, A.
\newblock Longformer: The long-document transformer.
\newblock \emph{arXiv preprint arXiv:2004.05150}, 2020.

\bibitem[Cand{\`e}s et~al.(2011)Cand{\`e}s, Li, Ma, and Wright]{robustpca}
Cand{\`e}s, E.~J., Li, X., Ma, Y., and Wright, J.
\newblock Robust principal component analysis?
\newblock \emph{Journal of the ACM (JACM)}, 58\penalty0 (3):\penalty0 1--37,
  2011.

\bibitem[Chen et~al.(2021)Chen, Dao, Winsor, Song, Rudra, and
  R\'{e}]{chen2021scatterbrain}
Chen, B., Dao, T., Winsor, E., Song, Z., Rudra, A., and R\'{e}, C.
\newblock Scatterbrain: Unifying sparse and low-rank attention.
\newblock In \emph{Advances in Neural Information Processing Systems
  (NeurIPS)}, 2021.

\bibitem[Choromanski et~al.(2021)Choromanski, Likhosherstov, Dohan, Song, Gane,
  Sarlos, Hawkins, Davis, Mohiuddin, Kaiser, Belanger, Colwell, and
  Weller]{choromanski2020rethinking}
Choromanski, K.~M., Likhosherstov, V., Dohan, D., Song, X., Gane, A., Sarlos,
  T., Hawkins, P., Davis, J.~Q., Mohiuddin, A., Kaiser, L., Belanger, D.~B.,
  Colwell, L.~J., and Weller, A.
\newblock Rethinking attention with performers.
\newblock In \emph{International Conference on Learning Representations
  (ICLR)}, 2021.

\bibitem[Coifman \& Maggioni(2006)Coifman and Maggioni]{coifman2006diffusion}
Coifman, R.~R. and Maggioni, M.
\newblock Diffusion wavelets.
\newblock \emph{Applied and computational harmonic analysis}, 21\penalty0
  (1):\penalty0 53--94, 2006.

\bibitem[d'Ascoli et~al.(2021)d'Ascoli, Touvron, Leavitt, Morcos, Biroli, and
  Sagun]{dascoli2021convit}
d'Ascoli, S., Touvron, H., Leavitt, M., Morcos, A., Biroli, G., and Sagun, L.
\newblock Convit: Improving vision transformers with soft convolutional
  inductive biases.
\newblock In \emph{International Conference on Machine Learning (ICML)}, 2021.

\bibitem[Daubechies(1992)]{daubechies1992ten}
Daubechies, I.
\newblock \emph{Ten Lectures on Wavelets}.
\newblock CBMS-NSF Regional Conference Series in Applied Mathematics. Society
  for Industrial and Applied Mathematics, 1992.

\bibitem[Devlin et~al.(2019)Devlin, Chang, Lee, and Toutanova]{devlin2018bert}
Devlin, J., Chang, M.-W., Lee, K., and Toutanova, K.
\newblock {BERT}: Pre-training of deep bidirectional transformers for language
  understanding.
\newblock In \emph{Conference of the North American Chapter of the Association
  for Computational Linguistics (NAACL)}, 2019.

\bibitem[Donoho(1995)]{donoho1995noising}
Donoho, D.~L.
\newblock De-noising by soft-thresholding.
\newblock \emph{IEEE Transactions on Information Theory}, 41\penalty0
  (3):\penalty0 613--627, 1995.

\bibitem[Donoho \& Johnstone(1995)Donoho and Johnstone]{donoho1995adapting}
Donoho, D.~L. and Johnstone, I.~M.
\newblock Adapting to unknown smoothness via wavelet shrinkage.
\newblock \emph{Journal of the American Statistical Association}, 90\penalty0
  (432):\penalty0 1200--1224, 1995.

\bibitem[Gavish et~al.(2010)Gavish, Nadler, and Coifman]{gavish2010multiscale}
Gavish, M., Nadler, B., and Coifman, R.~R.
\newblock Multiscale wavelets on trees, graphs and high dimensional data:
  Theory and applications to semi supervised learning.
\newblock In \emph{The International Conference on Machine Learning (ICML)},
  2010.

\bibitem[Hackbusch(1985)]{multigrid}
Hackbusch, W.
\newblock \emph{Multi-Grid Methods and Applications}, volume~4.
\newblock Springer Science \& Business Media, 01 1985.

\bibitem[Hammond et~al.(2011)Hammond, Vandergheynst, and
  Gribonval]{hammond2011wavelets}
Hammond, D.~K., Vandergheynst, P., and Gribonval, R.
\newblock Wavelets on graphs via spectral graph theory.
\newblock \emph{Applied and Computational Harmonic Analysis}, 30\penalty0
  (2):\penalty0 129--150, 2011.

\bibitem[Hu(2015)]{nuclear_norm_f_norm_inequ}
Hu, S.
\newblock Relations of the nuclear norm of a tensor and its matrix flattenings.
\newblock \emph{Linear Algebra and its Applications}, 478:\penalty0 188--199,
  2015.

\bibitem[Ithapu et~al.(2017)Ithapu, Kondor, Johnson, and
  Singh]{ithapu2017incremental}
Ithapu, V.~K., Kondor, R., Johnson, S.~C., and Singh, V.
\newblock The incremental multiresolution matrix factorization algorithm.
\newblock In \emph{Proceedings of the IEEE Conference on Computer Vision and
  Pattern Recognition}, 2017.

\bibitem[Kitaev et~al.(2020)Kitaev, Kaiser, and Levskaya]{Kitaev2020ReformerTE}
Kitaev, N., Kaiser, L., and Levskaya, A.
\newblock Reformer: The efficient transformer.
\newblock In \emph{International Conference on Learning Representations
  (ICLR)}, 2020.

\bibitem[Kondor et~al.(2014)Kondor, Teneva, and
  Garg]{kondor2014multiresolution}
Kondor, R., Teneva, N., and Garg, V.
\newblock Multiresolution matrix factorization.
\newblock In \emph{International Conference on Machine Learning (ICML)}, 2014.

\bibitem[Kovaleva et~al.(2019)Kovaleva, Romanov, Rogers, and
  Rumshisky]{Kovaleva2019}
Kovaleva, O., Romanov, A., Rogers, A., and Rumshisky, A.
\newblock Revealing the dark secrets of bert.
\newblock In \emph{Conference on Empirical Methods in Natural Language
  Processing (EMNLP)}, 2019.

\bibitem[Krizhevsky et~al.(2009)Krizhevsky, Hinton,
  et~al.]{krizhevsky2009learning}
Krizhevsky, A., Hinton, G., et~al.
\newblock Learning multiple layers of features from tiny images.
\newblock 2009.

\bibitem[Lee \& Nadler(2007)Lee and Nadler]{lee2007treelets}
Lee, A.~B. and Nadler, B.
\newblock Treelets: A tool for dimensionality reduction and multi-scale
  analysis of unstructured data.
\newblock In \emph{International Conference on Artificial Intelligence and
  Statistics}, 2007.

\bibitem[Lee-Thorp et~al.(2021)Lee-Thorp, Ainslie, Eckstein, and
  Ontanon]{lee2021fnet}
Lee-Thorp, J., Ainslie, J., Eckstein, I., and Ontanon, S.
\newblock Fnet: Mixing tokens with fourier transforms.
\newblock \emph{arXiv preprint arXiv:2105.03824}, 2021.

\bibitem[Liu et~al.(2019)Liu, Ott, Goyal, Du, Joshi, Chen, Levy, Lewis,
  Zettlemoyer, and Stoyanov]{liu2019roberta}
Liu, Y., Ott, M., Goyal, N., Du, J., Joshi, M., Chen, D., Levy, O., Lewis, M.,
  Zettlemoyer, L., and Stoyanov, V.
\newblock Roberta: A robustly optimized bert pretraining approach.
\newblock \emph{arXiv preprint arXiv:1907.11692}, 2019.

\bibitem[Lu et~al.(2021)Lu, Yao, Zhang, Zhu, Xu, Gao, Xu, Xiang, and
  Zhang]{SOFT}
Lu, J., Yao, J., Zhang, J., Zhu, X., Xu, H., Gao, W., Xu, C., Xiang, T., and
  Zhang, L.
\newblock Soft: Softmax-free transformer with linear complexity.
\newblock In \emph{Advances in Neural Information Processing Systems
  (NeurIPS)}, 2021.

\bibitem[Mallat(1999)]{mallat1999wavelet}
Mallat, S.
\newblock \emph{A Wavelet Tour of Signal Processing}.
\newblock Elsevier, 1999.

\bibitem[Peng et~al.(2021)Peng, Pappas, Yogatama, Schwartz, Smith, and
  Kong]{peng2021rfa}
Peng, H., Pappas, N., Yogatama, D., Schwartz, R., Smith, N., and Kong, L.
\newblock Random feature attention.
\newblock In \emph{International Conference on Learning Representations
  (ICLR)}, 2021.

\bibitem[Peng et~al.(2018)Peng, Lu, Yi, and Tang]{nuclear_norm_f_norm_relation}
Peng, X., Lu, C., Yi, Z., and Tang, H.
\newblock Connections between nuclear-norm and frobenius-norm-based
  representations.
\newblock \emph{IEEE Transactions on Neural Networks and Learning Systems},
  29\penalty0 (1):\penalty0 218--224, 2018.

\bibitem[Russakovsky et~al.(2015)Russakovsky, Deng, Su, Krause, Satheesh, Ma,
  Huang, Karpathy, Khosla, Bernstein, Berg, and Fei-Fei]{imagenet}
Russakovsky, O., Deng, J., Su, H., Krause, J., Satheesh, S., Ma, S., Huang, Z.,
  Karpathy, A., Khosla, A., Bernstein, M., Berg, A.~C., and Fei-Fei, L.
\newblock {ImageNet Large Scale Visual Recognition Challenge}.
\newblock In \emph{International Journal of Computer Vision (IJCV)}, 2015.

\bibitem[Saad(2003)]{saad_2003}
Saad, Y.
\newblock \emph{Iterative Methods for Sparse Linear Systems}.
\newblock Society for Industrial and Applied Mathematics, second edition, 2003.

\bibitem[Simic(2009)]{jensen_gap}
Simic, S.
\newblock Jensen’s inequality and new entropy bounds.
\newblock \emph{Applied Mathematics Letters}, 22\penalty0 (8):\penalty0
  1262--1265, 2009.

\bibitem[Tay et~al.(2021)Tay, Dehghani, Abnar, Shen, Bahri, Pham, Rao, Yang,
  Ruder, and Metzler]{tay2020long}
Tay, Y., Dehghani, M., Abnar, S., Shen, Y., Bahri, D., Pham, P., Rao, J., Yang,
  L., Ruder, S., and Metzler, D.
\newblock Long range arena : A benchmark for efficient transformers.
\newblock In \emph{International Conference on Learning Representations
  (ICLR)}, 2021.

\bibitem[Trinh \& Le(2018)Trinh and Le]{trinh2019simple}
Trinh, T.~H. and Le, Q.~V.
\newblock A simple method for commonsense reasoning.
\newblock \emph{arXiv preprint arXiv:1806.02847}, 2018.

\bibitem[Vaswani et~al.(2017)Vaswani, Shazeer, Parmar, Uszkoreit, Jones, Gomez,
  Kaiser, and Polosukhin]{vaswani2017attention}
Vaswani, A., Shazeer, N., Parmar, N., Uszkoreit, J., Jones, L., Gomez, A.~N.,
  Kaiser, L.~u., and Polosukhin, I.
\newblock Attention is all you need.
\newblock In \emph{Advances in Neural Information Processing Systems
  (NeurIPS)}, 2017.

\bibitem[Wang et~al.(2020)Wang, Li, Khabsa, Fang, and Ma]{Wang2020LinformerSW}
Wang, S., Li, B., Khabsa, M., Fang, H., and Ma, H.
\newblock Linformer: Self-attention with linear complexity.
\newblock \emph{arXiv preprint arXiv:2006.04768}, 2020.

\bibitem[Welbl et~al.(2018)Welbl, Stenetorp, and Riedel]{wikihop}
Welbl, J., Stenetorp, P., and Riedel, S.
\newblock Constructing datasets for multi-hop reading comprehension across
  documents.
\newblock \emph{Transactions of the Association for Computational Linguistics},
  6:\penalty0 287--302, 2018.

\bibitem[Williams et~al.(2018)Williams, Nangia, and Bowman]{williams2018broad}
Williams, A., Nangia, N., and Bowman, S.
\newblock A broad-coverage challenge corpus for sentence understanding through
  inference.
\newblock In \emph{Conference of the North American Chapter of the Association
  for Computational Linguistics (NAACL)}, 2018.

\bibitem[Wolf et~al.(2020)Wolf, Debut, Sanh, Chaumond, Delangue, Moi, Cistac,
  Rault, Louf, Funtowicz, Davison, Shleifer, von Platen, Ma, Jernite, Plu, Xu,
  Scao, Gugger, Drame, Lhoest, and Rush]{hugginface}
Wolf, T., Debut, L., Sanh, V., Chaumond, J., Delangue, C., Moi, A., Cistac, P.,
  Rault, T., Louf, R., Funtowicz, M., Davison, J., Shleifer, S., von Platen,
  P., Ma, C., Jernite, Y., Plu, J., Xu, C., Scao, T.~L., Gugger, S., Drame, M.,
  Lhoest, Q., and Rush, A.~M.
\newblock Transformers: State-of-the-art natural language processing.
\newblock In \emph{Conference on Empirical Methods in Natural Language
  Processing (EMNLP)}, 2020.

\bibitem[Xiong et~al.(2021)Xiong, Zeng, Chakraborty, Tan, Fung, Li, and
  Singh]{xiong2021nystromformer}
Xiong, Y., Zeng, Z., Chakraborty, R., Tan, M., Fung, G., Li, Y., and Singh, V.
\newblock Nyströmformer: A nyström-based algorithm for approximating
  self-attention.
\newblock In \emph{Proceedings of the AAAI Conference on Artificial
  Intelligence}, 2021.

\bibitem[Xu et~al.(2021)Xu, Zhang, Zhang, and Tao]{xu2021vitae}
Xu, Y., Zhang, Q., Zhang, J., and Tao, D.
\newblock Vitae: Vision transformer advanced by exploring intrinsic inductive
  bias.
\newblock In \emph{Advances in Neural Information Processing Systems
  (NeurIPS)}, 2021.

\bibitem[Zaheer et~al.(2020)Zaheer, Guruganesh, Dubey, Ainslie, Alberti,
  Ontanon, Pham, Ravula, Wang, Yang, and Ahmed]{zaheer2020big}
Zaheer, M., Guruganesh, G., Dubey, K.~A., Ainslie, J., Alberti, C., Ontanon,
  S., Pham, P., Ravula, A., Wang, Q., Yang, L., and Ahmed, A.
\newblock Big bird: Transformers for longer sequences.
\newblock In \emph{Advances in Neural Information Processing Systems
  (NeurIPS)}, 2020.

\bibitem[Zellers et~al.(2019)Zellers, Holtzman, Rashkin, Bisk, Farhadi,
  Roesner, and Choi]{zellers2019defending}
Zellers, R., Holtzman, A., Rashkin, H., Bisk, Y., Farhadi, A., Roesner, F., and
  Choi, Y.
\newblock Defending against neural fake news.
\newblock In \emph{Advances in Neural Information Processing Systems
  (NeurIPS)}, 2019.

\bibitem[Zeng et~al.(2021)Zeng, Xiong, Ravi, Acharya, Fung, and Singh]{yoso}
Zeng, Z., Xiong, Y., Ravi, S., Acharya, S., Fung, G.~M., and Singh, V.
\newblock You only sample (almost) once: Linear cost self-attention via
  bernoulli sampling.
\newblock In \emph{International Conference on Machine Learning (ICML)}, 2021.

\bibitem[Zhu et~al.(2015)Zhu, Kiros, Zemel, Salakhutdinov, Urtasun, Torralba,
  and Fidler]{zhu2015aligning}
Zhu, Y., Kiros, R., Zemel, R., Salakhutdinov, R., Urtasun, R., Torralba, A.,
  and Fidler, S.
\newblock Aligning books and movies: Towards story-like visual explanations by
  watching movies and reading books.
\newblock In \emph{International Conference on Computer Vision (ICCV)}, 2015.

\bibitem[Zhu \& Soricut(2021)Zhu and Soricut]{htransformer1d}
Zhu, Z. and Soricut, R.
\newblock {H}-transformer-1{D}: Fast one-dimensional hierarchical attention for
  sequences.
\newblock In \emph{Annual Meeting of the Association for Computational
  Linguistics}, 2021.

\end{thebibliography}
\bibliographystyle{icml2022}

\newpage
\appendix
\onecolumn

\section{Appendix}

The appendix includes more details regarding the formulation, analysis, and experiments.

\subsection{Visualizing Approximation Procedure in Linear Scale}
\label{sec:approx_procedure_linear}

\begin{figure*}[!htbp]
\centering
\includegraphics[width=6.5in]{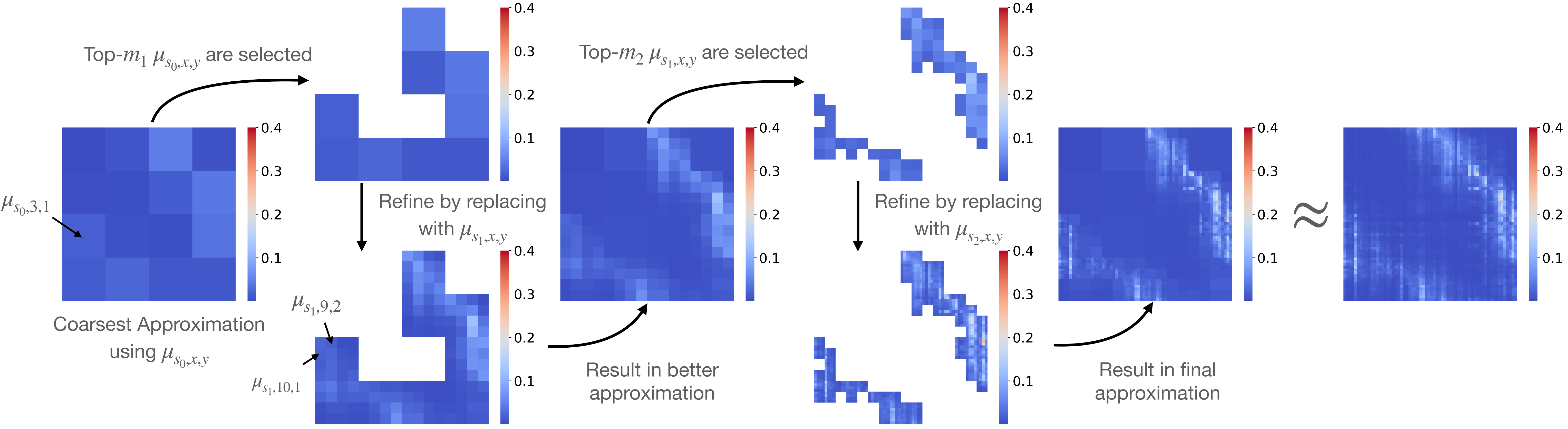}
\caption{Illustration of approximation scheme for $R = \{16, 4, 1\}$ in linear scale. }
\label{fig:multires_approx_linear}
\end{figure*}

We also show a visualization of our approximation procedure using a linear scale in Fig. \ref{fig:multires_approx_linear}. This figure gives a better illustration of how approximation quality increases as approximation procedure proceeds. 

\subsection{Link to Sparsity and Low Rank}
\label{sec:appdenix_sp_lr_dec}



\label{sec:sparse_low_rank_link}

\begin{figure}[!htbp]
\centering
\includegraphics[width=4in]{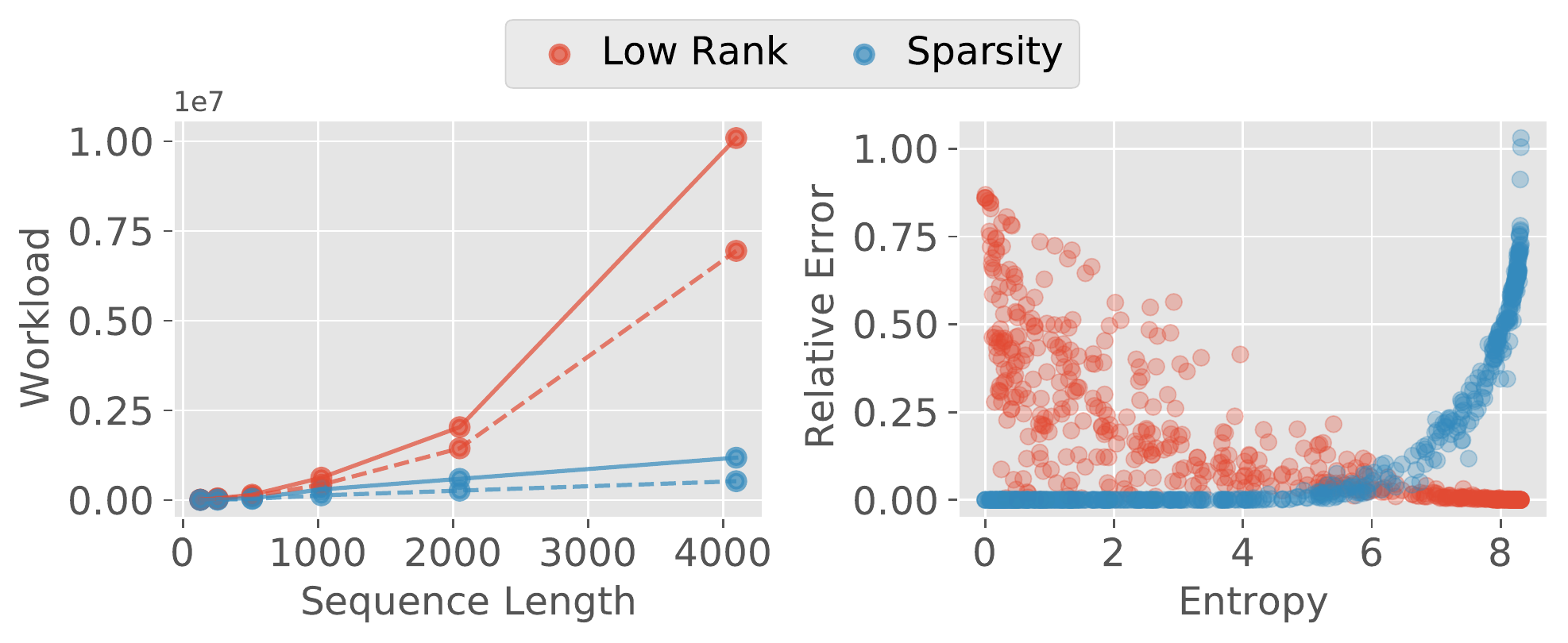}
\caption{The solid and dashed lines in the left plot is the theoretical workload, without considering the overhead, necessary to get relative error less than $0.05$ and $0.1$, respectively, for different sequence length. Ideally, the workload should be linear to sequence length. The right plot is the approximation error vs entropy of self-attention matrices, similar to \cite{chen2021scatterbrain}, at $1/4$ of workload for standard self-attention (keeping $25$\% of rank and nonzero entries for low rank and sparsity, respectively). Entropy of softmax is used as a proxy for the spread of attention. Relative error is defines as $||\hat{D} \hat{\mathcal{A}} V - D \mathcal{A} V||_F / ||D \mathcal{A} V||_F$}
\label{fig:lr_vs_sp}
\end{figure}

Low rank and sparsity are two popular directions for efficient self-attention. To explore the potential of these two types of approximations, we set aside the efficiency consideration and use the best possible methods for each type of approximation. Specifically, subject to $||\hat{\mathcal{A}} - \mathcal{A}||_F \leq \epsilon$, we use sparse approximation which minimizes $||\hat{\mathcal{A}}||_0$ by finding support on the largest entries of $\mathcal{A}$, and low rank approximation which minimizes $\text{rank}(\hat{\mathcal{A}})$ via truncated SVD. As shown in Fig. \ref{fig:lr_vs_sp}, these two types of methods are limited for approximating self-attention. The low rank method requires superlinear cost to maintain the approximation accuracy and fails when the entropy of self-attention is smaller. In many cases, sparse approximation is sufficient, but in some cases when the self-attention matrices are less sparse and have larger entropy, the sparse approximation would fail as well. 
This motivates the use of sparse + low rank approximation. 
It can be achieved via robust PCA which decomposes approximation to $\hat{\mathcal{A}} = S + L$ by solving an optimization objective $||S||_0 + \lambda \ \text{rank}(L)$. A convex relaxation $||S||_1 + \lambda ||L||_*$ \cite{robustpca} is used to make the optimization tractable, but the cost of finding a good solution is still more than $O(n^2)$, which is not suitable for efficient self-attention. 
Scatterbrain \cite{chen2021scatterbrain} proposes to combine an existing sparse attention with an existing low rank attention to obtain a sparse + low rank approximation and avoid the expensive cost of robust PCA. 

Interestingly, a special form of our work offers an alternative to Scatterbrain's approach for sparse + low rank approximation. When $R = \{b, 1\}$ for some $b$, our MRA-2 can be viewed as a sparse + low rank approximation. Specifically, let 
\begin{equation}
\hat{\mathcal{A}}_r = \sum_{B^b_{x, y} \in \mathcal{J}} \mu_{b,x,y} B^b_{x,y}
\label{eq:single_res_components}
\end{equation}
for a resolution $b$, then $\hat{\mathcal{A}}_1 + \hat{\mathcal{A}}_b$ serves as a reasonably good solution for a relaxed version of sparse and low rank decomposition. 

Let us consider an alternative relaxation of robust PCA objective, 
\begin{equation}
|| S ||_0 + \lambda || L ||_F
\label{eq:relaxed_rpca}
\end{equation}
Note that $|| L ||_F$ is easier to compute. And we have $||L||_* \leq \sqrt{n} || L ||_F$ \cite{nuclear_norm_f_norm_inequ}. In fact, \citet{nuclear_norm_f_norm_relation} shows that solutions obtained by minimizing $||L||_*$ and $|| L ||_F$ are two solutions of a low rank recovery problem and are identical in some situations. The optimal solution for objective \eqref{eq:relaxed_rpca} can be easily obtained. For $\epsilon = 0$, there exists a $m$ such that the optimal $S$ has support on the $m$ largest entries of $\mathcal{A}$, and $L = \mathcal{A} - S$. However, for practical use, the recovered sparsity cannot be efficiently used on a GPU due to scattered support. Further, the complexity is $O(n^2)$ since we found that we still need to find the largest entries of $\mathcal{A}$. Suppose we restrict $S$ to be a block sparse matrix, namely, supported on a subset of $\{\text{supp}(B^{b}_{x, y})\}_{x, y}$, then a GPU can exploit this block sparsity structure to significantly accelerate computation. Further, the optimal $S$ is supported on the regions with the largest $\mu_{b, x, y}'$ defined as
\begin{equation}
\begin{split}
\mu_{b, x, y}' &= \frac{1}{b^2} \sum_{(i, j) \in \supp(B^b_{x,y})} \exp(2 \mathcal{P}_{i,j}) \\
\end{split}
\label{eq:optimal_block_scores}
\end{equation}
which is similar to \eqref{eq:average_of_exp}. As a result, similar to approximating \eqref{eq:average_of_exp} with \eqref{eq:exp_of_average}, we use the lower bound $\mu_{b, x, y}^2$ as an proxy for \eqref{eq:optimal_block_scores}. Then, the cost of locating support blocks is only $O(n^2 / b^2)$. Consequently, the resulting solution $S$ is supported on the regions with the largest $\mu_{b, x, y}$. Note that this $S$ is exactly $\hat{\mathcal{A}}_1$ with an appropriate budget $m$. And the $\hat{\mathcal{A}}_b$ is a reasonable solution for $L$ since $|| \hat{\mathcal{A}}_b ||_F$ is small and $\text{rank}(\hat{\mathcal{A}}_b) \leq n / b$.

\begin{figure}[!htbp]
\centering
\includegraphics[width=4in]{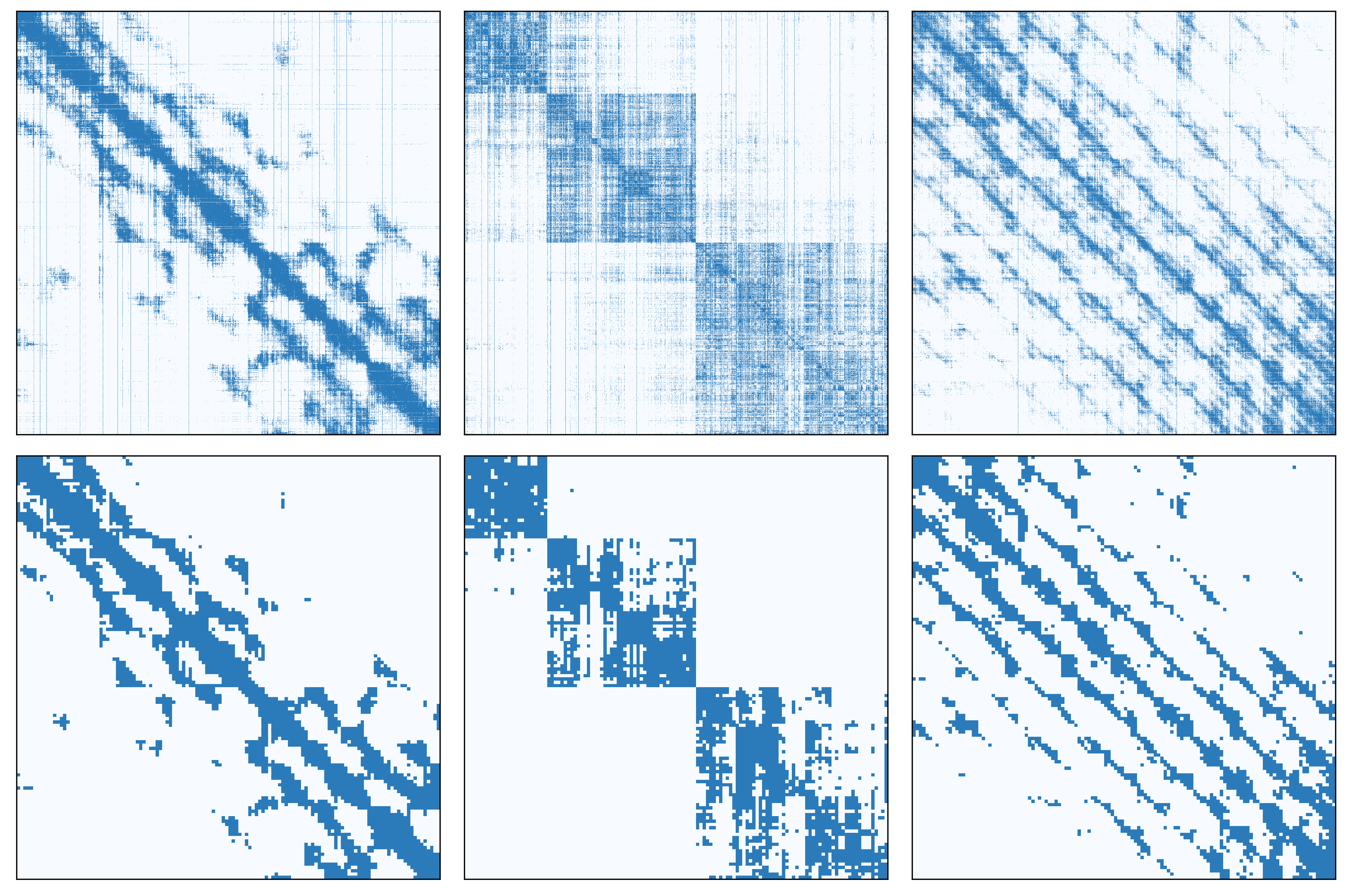}
\caption{The top 3 plots are the optimal supports for 3 typical self-attention matrices at $80$\% sparsity. The bottom 3 plots are supports found via our MRA-2 at $80$\% sparsity. }
\label{fig:sp_pattern}
\end{figure}

We empirically evaluate the quality of this sparse solution. \cite{Kovaleva2019} showed that the BERT model \cite{devlin2018bert} has multiple self-attention patterns capturing different semantic information. We investigate the types of possible self-attention of a pretrained RoBERTa model \cite{liu2019roberta}. And we show the optimal sparsity supports for self-attention matrices generated from RoBERTa-base with 4096 sequence in the top plots of Fig. \ref{fig:sp_pattern}. Our MRA-2, as shown in the bottom plots of Fig. \ref{fig:sp_pattern}, is able to find a reasonably good sparse solution for \eqref{eq:relaxed_rpca}. Notice that while many self-attention matrices tend out to be diagonally banded matrices, which Longformer and Big Bird can approximate well, they are not the only possible structure. Diagonally banded structure is not sufficient to approximate the last two self-attention patterns well. 

\subsection{Analysis}
\label{sec:appendix_analysis}

In this section, we discuss in more detail some manipulations described in the main paper.

\begin{observation}
We can re-write $\hat{\mathcal{A}}^* = \sum_{B^{s}_{x, y} \in \mathcal{J}} \alpha^{s}_{x, y} B^{s}_{x, y}$ as $(\hat{\mathcal{A}}^*)_{i,j} = \mu_{s, x, y}^*$ where $(s, x, y)$ is the index of $B^{s}_{x, y} \in \mathcal{J}$ that has the smallest support region and is supported on $(i, j)$.
\end{observation}

\begin{proof}
We describe the details next. 
First, notice that at each scale $s$,
\begin{equation}
E_{s/2} = E_{s} - \sum_{B^{s}_{x, y} \in \mathcal{J}} \alpha^{s}_{x, y} B^{s}_{x, y}
\end{equation}
If $(i, j) \not\in \supp(B^{s}_{x, y})$ for all $B^{s}_{x, y} \in \mathcal{J}$, then $(\sum_{B^{s}_{x, y} \in \mathcal{J}} \alpha^{s}_{x, y} B^{s}_{x, y})_{i,j} = 0$ and thus $(E_{s/2})_{i,j} = (E_{s})_{i,j}$. Further, at each scale $s$, the supports of $B^{s}_{x,y}$ are disjoint, and there is exactly one $B^{s}_{x,y}$ whose support includes $(i, j)$. Thus, if $(E_{s/2})_{i,j} = (E_{s})_{i,j}$, then $(E_{s/2})_{i',j'} = (E_{s})_{i',j'}$ for all $(i', j') \in \supp(B^{s}_{x,y})$. Also, for scale $s' < s$, if $B^{s'}_{x',y'}$ is supported on $(i, j)$, then $\supp(B^{s'}_{x',y'}) \subseteq \supp(B^{s}_{x,y})$. These observations are useful for the derivation below. 

Consider approximation on the $(i, j)$ entry of $\hat{\mathcal{A}}^*$, and let $B^{s_1}_{x_1,y_1}, \cdots, B^{s_k}_{x_k,y_k} \in \mathcal{J}$ be all components that are supported on $(i, j)$ and $s_1 < \cdots < s_k$. 
Then, $E_n = \mathcal{A}$. Since $B^{s_k}_{x_k,y_k}$ has the largest $s_k$ and is supported on $(i, j)$, by the above observations, for all $(i', j') \in \supp(B^{s_k}_{x_k,y_k})$, 
\begin{equation}
(E_{s_k})_{i',j'} = (E_{2 s_k})_{i',j'} = \cdots = (E_{n})_{i',j'} = \mathcal{A}_{i',j'}
\end{equation}
Therefore, 
\begin{equation}
\alpha^{s_k}_{x_k, y_k} = \frac{1}{s_k^2} \inner{B^{s_k}_{x_k, y_k}, E_{s_k}} = \frac{1}{s_k^2} \inner{B^{s_k}_{x_k, y_k}, \mathcal{A}} = \mu_{s_k, x_k, y_k}^*
\end{equation}
Then, 
\begin{equation}
(E_{s_k/2})_{i',j'} = (E_{s_k})_{i',j'} - (\sum_{B^{s_k}_{x_k, y_k} \in \mathcal{J}} (\alpha^{s_k}_{x_k, y_k})' B^{s_k}_{x_k, y_k})_{i',j'} = \mathcal{A}_{i',j'} - \mu_{s_k, x_k, y_k}^*
\end{equation}
Assume for all $(i', j') \in \supp(B^{s_{m+1}}_{x_{m+1},y_{m+1}})$, 
\begin{equation}
(E_{s_{m+1}/2})_{i',j'} = \mathcal{A}_{i',j'} - \mu_{s_{m+1}, x_{m+1}, y_{m+1}}^*
\end{equation}
Then, again by the above observations, 
\begin{equation}
(E_{s_m})_{i',j'} = \cdots = (E_{s_{m+1}/2})_{i',j'} = \mathcal{A}_{i',j'} - \mu_{s_{m+1}, x_{m+1}, y_{m+1}}^*
\end{equation}
Therefore, 
\begin{equation}
\begin{split}
\alpha^{s_m}_{x_m, y_m} &= \frac{1}{s_m^2} \inner{B^{s_m}_{x_m, y_m}, E_{s_m}} \\
&= \frac{1}{s_m^2} \sum_{(i'',j'') \in \supp(B^{s_m}_{x_m, y_m})} (\mathcal{A}_{i'',j''} - \mu_{s_{m+1}, x_{m+1}, y_{m+1}}^*) \\
&= \frac{1}{s_m^2} \sum_{(i'',j'') \in \supp(B^{s_m}_{x_m, y_m})} \mathcal{A}_{i'',j''} - \mu_{s_{m+1}, x_{m+1}, y_{m+1}}^* \\
&= \mu_{s_m, x_m, y_m}^* - \mu_{s_{m+1}, x_{m+1}, y_{m+1}}^* \\
\end{split}
\end{equation}
Then, 
\begin{equation}
\begin{split}
(E_{s_m/2})_{i',j'} &= (E_{s_m})_{i',j'} - (\sum_{B^{s_m}_{x_m, y_m} \in \mathcal{J}} (\alpha^{s_m}_{x_m, y_m})' B^{s_m}_{x_m, y_m})_{i',j'} \\
&= \mathcal{A}_{i,j} - \mu_{s_{m+1}, x_{m+1}, y_{m+1}}^* - (\mu_{s_m, x_m, y_m}^* - \mu_{s_{m+1}, x_{m+1}, y_{m+1}}^*) \\
&= \mathcal{A}_{i,j} - \mu_{s_m, x_m, y_m}^* \\
\end{split}
\end{equation}
By induction, for all $(i', j') \in \supp(B^{s_m}_{x_m,y_m})$, 
\begin{equation}
\begin{split}
(E_{s_m/2})_{i',j'} = \mathcal{A}_{i',j'} - \mu_{s_m, x_m, y_m}^* \\
\end{split}
\end{equation}
holds for all $s_m$. Further, we know for $m < k$, 
\begin{equation}
\begin{split}
\alpha^{s_k}_{x_k, y_k} &= \mu_{s_k, x_k, y_k}^* \\
\alpha^{s_m}_{x_m, y_m} &= \mu_{s_m, x_m, y_m}^* - \mu_{s_{m+1}, x_{m+1}, y_{m+1}}^* \\
\end{split}
\end{equation}
Finally, 
\begin{equation}
\begin{split}
(\hat{\mathcal{A}}^*)_{i,j} = (\sum_{B^{s}_{x, y} \in \mathcal{J}} \alpha^{s}_{x, y} B^{s}_{x, y})_{i,j} = \mu_{s_k, x_k, y_k}^* + \sum_{m=1}^{k-1} (\mu_{s_m, x_m, y_m}^* - \mu_{s_{m+1}, x_{m+1}, y_{m+1}}^*) = \mu_{s_1, x_1, y_1}^*
\end{split}
\end{equation}
And, $(s_1, x_1, y_1)$ is the index of $B^{s_1}_{x_1, y_1} \in \mathcal{J}$ that has the smallest support region and is supported on $(i, j)$. This completes the proof. 

\end{proof}

\begin{lemma}
Assume for all $(i_1, j_1), (i_2, j_2) \in \supp(B^{s}_{x,y})$, $||Q_{i_1}||_p, ||Q_{i_2}||_p, ||K_{j_1}||_p, ||K_{j_2}||_p \leq \beta_1$ and $||Q_{i_1} - Q_{i_2}||_q, ||K_{j_1} - K_{j_2}||_q \leq \beta_2$ where $\frac{1}{p} + \frac{1}{q} = 1$, then $\mathcal{P}_{i,j} \in [a, a + r]$ where $r = 2 \beta_1 \beta_2$ for all $(i, j) \in \supp(B^{s}_{x,y})$ and some $a$, and 
\begin{equation*}
0 \leq \mu_{s, x, y}^* - \mu_{s, x, y} \leq C_r \mu_{s, x, y}
\end{equation*}
where $C_r = 1 + \exp(r) - 2 \exp(r/2)$. 
\end{lemma}

\begin{proof}

\begin{equation}
\begin{split}
\mathcal{P}_{i_1, j_1} - \mathcal{P}_{i_2, j_2} &= Q_{i_1} K_{j_1}^T - Q_{i_2} K_{j_2}^T \\
&= (Q_{i_1} - Q_{i_2} + Q_{i_2}) K_{j_1}^T - Q_{i_2} K_{j_2}^T \\
&= (Q_{i_1} - Q_{i_2}) K_{j_1}^T + Q_{i_2} (K_{j_1} - K_{j_2})^T 
\end{split}
\end{equation}
Then by H\"{o}lder's inequality,
\begin{equation}
\begin{split}
\mathcal{P}_{i_1, j_1} - \mathcal{P}_{i_2, j_2} &\leq ||Q_{i_1} - Q_{i_2}||_q ||K_{j_1}||_p + ||Q_{i_2}||_p ||K_{j_1} - K_{j_2}||_q \leq 2 \beta_1 \beta_2
\end{split}
\end{equation}
Therefore,  $\mathcal{P}_{i,j} \in [a, a + r]$ where $r = 2 \beta_1 \beta_2$ and $a = \min_{(i',j') \in \supp(B^s_{x,y})} \mathcal{P}_{i',j'}$. Then, by Jensen's inequality, 
\begin{equation}
\mu_{s,x,y} = \exp(\frac{1}{s^2} \sum_{(i,j) \in \supp(B^s_{x,y})} \mathcal{P}_{i, j}) \leq \frac{1}{s^2} \sum_{(i,j) \in \supp(B^s_{x,y})} \exp(\mathcal{P}_{i,j}) = \mu_{s,x,y}^*
\end{equation}
Also, by \citet{jensen_gap}, for a convex function $f$ and $x_i \in [a, b]$ for all $x_i$, 
\begin{equation}
\frac{1}{k} \sum_{i=1}^k f(x_i) - f(\frac{1}{k} \sum_{i=1}^k x_i) \leq f(a) + f(b) - 2 f(\frac{a + b}{2})
\end{equation}
Therefore, 
\begin{equation}
\begin{split}
\mu_{s,x,y}^* - \mu_{s,x,y} &\leq \exp(a) + \exp(a + r) - 2 \exp(a + \frac{r}{2}) \leq (1 + \exp(r) - 2 \exp(\frac{r}{2})) \mu_{s,x,y} 
\label{eq:jensen_gap_C_r}
\end{split}
\end{equation}
Denote $C_r = 1 + \exp(r) - 2 \exp(\frac{r}{2})$. 

\end{proof}

\begin{proposition}
Let $R = \{b, 1\}$ for some $b$ and $\delta$ be the $m_1$-th largest $\mu_{b, x, y}$, assume for all $(i,j) \in \supp(B^{b}_{x,y})$, $\mathcal{P}_{i,j} \in [a, a + r]$ for some $a$ and $r > 0$, then 
\begin{equation*}
\begin{split}
\frac{||\hat{\mathcal{A}} - \mathcal{A}||_F}{||\mathcal{A}||_F} &\leq \sqrt{\frac{(n^2 - m_1 b^2) C_{2r} \delta^2 }{ \sum_{i,j=1}^n \exp(2 \mathcal{P}_{i,j})}}
\end{split}
\end{equation*}
where $C_{2r} = 1 + \exp(2r) - 2 \exp(r)$
\end{proposition}

\begin{proof}

Similar to \eqref{eq:jensen_gap_C_r}, for $\mu_{s,x,y}'$ defined in \eqref{eq:optimal_block_scores}
\begin{equation}
\begin{split}
\mu_{s,x,y}' - \mu_{s,x,y}^2 &\leq \exp(2a) + \exp(2a + 2r) - 2 \exp(2a + r) \leq (1 + \exp(2r) - 2 \exp(r)) \mu_{s,x,y}^2
\end{split}
\end{equation}
Denote $C_{2r} = 1 + \exp(2r) - 2 \exp(r)$, Then, the error by approximating a region $\supp(B^{s}_{x,y})$ with lower resolution $\mu_{s,x,y}$ is
\begin{equation}
\begin{split}
\sum_{(i,j) \in \supp(B^{s}_{x,y})} (\mu_{s,x,y} - \exp(\mathcal{P}_{i,j}))^2 &= \sum_{(i,j) \in \supp(B^{s}_{x,y})} (\mu_{s,x,y}^2 + \exp(2 \mathcal{P}_{i,j}) - 2 \mu_{s,x,y} \exp(\mathcal{P}_{i,j})) \\
&= s^2 \mu_{s,x,y}^2 + s^2 \mu_{s,x,y}' - 2 s^2 \mu_{s,x,y} \mu_{s,x,y}^* \\
&\leq s^2 \mu_{s,x,y}^2 + s^2 \mu_{s,x,y}' - 2 s^2 \mu_{s,x,y}^2 \\
&\leq C_{2r} s^2 \mu_{s,x,y}^2
\end{split}
\end{equation}

Since $\mu_{1,x,y} = \exp(\mathcal{P}_{x,y})$, the approximation error only comes from terms $B^s_{x,y} \in \mathcal{J}$ for $s > 1$. Therefore, 
\begin{equation}
\begin{split}
||\hat{\mathcal{A}} - \mathcal{A}||^2_F &= \sum_{B^b_{x,y} \in \mathcal{J}} \sum_{(i,j) \in \supp(B^b_{x,y})} (\mu_{b,x,y} - \exp(\mathcal{P}_{i,j}))^2 \\
&= \sum_{B^b_{x,y} \in \mathcal{J}} C_{2r} b^2 \mu_{b,x,y}^2 \\
&\leq \sum_{B^b_{x,y} \in \mathcal{J}} C_{2r} b^2 \delta^2 \\
&= (n^2 - m_1 b^2) C_{2r} \delta^2 \\
\end{split}
\end{equation}
Then, the relative error is
\begin{equation}
\begin{split}
\frac{||\hat{\mathcal{A}} - \mathcal{A}||_F}{||\mathcal{A}||_F} &\leq \sqrt{\frac{(n^2 - m_1 b^2) C_{2r} \delta^2 }{ \sum_{i,j=1}^n \exp(2 \mathcal{P}_{i,j})}}
\end{split}
\end{equation}

\end{proof}

\subsection{Experiments}
\label{sec:appendix_exp}

\textbf{Implementations}. We use Reformer, Longformer, and Big Bird implementations from HuggingFace \cite{hugginface}. Nystromformer, SOFT, and Scatterbrain implementations are from their official Github repo. Linformer implementation is from Nystromformer's repo. Performer and H-Transformer-1D implementations are from Lucidrains repo.

\begin{table*}[!htbp]
\begin{center}
\begin{tiny}
\begin{tabular}{lccc|ccc|ccc|ccc|ccc}
\toprule
Sequence Length & \multicolumn{3}{c}{ 256 } & \multicolumn{3}{c}{ 512 } & \multicolumn{3}{c}{ 1024 } & \multicolumn{3}{c}{ 2048 } & \multicolumn{3}{c}{ 4096 } \\
Method & Time & Memory & Error & Time & Memory & Error & Time & Memory & Error & Time & Memory & Error & Time & Memory & Error \\
\midrule
Transformer & 0.33 & 41.2 & 0.00 & 0.85 & 72.3 & 0.00 & 2.53 & 272.2 & 0.00 & 8.66 & 1085.8 & 0.00 & 30.67 & 4579.7 & 0.00  \\
\midrule
Nystromformer & 0.31 & 11.4 & 0.70 & 0.59 & 22.4 & 0.76 & 1.16 & 45.3 & 0.83 & 2.29 & 95.1 & 0.88 & 4.50 & 209.8 & 0.91  \\
 & 0.33 & 13.8 & 0.62 & 0.62 & 25.8 & 0.69 & 1.18 & 50.7 & 0.77 & 2.31 & 104.2 & 0.84 & 4.62 & 226.7 & 0.89  \\
 & 0.41 & 21.6 & 0.55 & 0.71 & 35.4 & 0.62 & 1.31 & 64.0 & 0.71 & 2.52 & 125.1 & 0.78 & 4.96 & 262.6 & 0.85  \\
 & - & - & - & 1.12 & 66.3 & 0.55 & 1.79 & 104.4 & 0.63 & 3.17 & 188.2 & 0.72 & 6.17 & 345.3 & 0.79  \\
 & - & - & - & - & - & - & 3.70 & 238.6 & 0.55 & 5.50 & 375.3 & 0.65 & 8.92 & 627.8 & 0.73  \\
 & - & - & - & - & - & - & - & - & - & 16.40 & 947.1 & 0.57 & 22.28 & 1573.1 & 0.66  \\
 & - & - & - & - & - & - & - & - & - & - & - & - & 91.82 & 4018.8 & 0.58  \\
\midrule
Linformer & 0.26 & 10.5 & 8.04 & 0.51 & 21.1 & 11.45 & 1.02 & 42.7 & 18.33 & 2.06 & 93.1 & 28.29 & 4.10 & 206.4 & 43.60  \\
 & 0.26 & 11.5 & 5.42 & 0.51 & 23.0 & 8.27 & 1.02 & 47.2 & 12.45 & 2.07 & 99.3 & 19.97 & 4.15 & 218.9 & 31.21  \\
 & 0.27 & 13.1 & 3.61 & 0.54 & 26.2 & 5.91 & 1.06 & 53.4 & 9.12 & 2.17 & 111.5 & 14.20 & 4.36 & 243.3 & 22.57  \\
 & 0.30 & 16.9 & 2.56 & 0.60 & 33.3 & 4.02 & 1.20 & 67.2 & 6.26 & 2.40 & 138.9 & 9.85 & 4.82 & 297.8 & 15.91  \\
 & 0.37 & 28.3 & 1.66 & 0.74 & 55.8 & 2.71 & 1.47 & 112.7 & 4.33 & 2.94 & 234.6 & 6.74 & 5.90 & 510.0 & 11.05  \\
 & - & - & - & 1.02 & 100.8 & 1.86 & 2.03 & 204.0 & 3.03 & 4.07 & 427.1 & 4.80 & 8.12 & 938.8 & 7.62  \\
 & - & - & - & - & - & - & 3.15 & 387.8 & 1.98 & 6.31 & 816.5 & 3.09 & 12.74 & 1814.4 & 5.18  \\
 & - & - & - & - & - & - & - & - & - & 11.15 & 1613.4 & 2.21 & 22.74 & 3637.9 & 3.54  \\
 & - & - & - & - & - & - & - & - & - & - & - & - & 43.32 & 7572.3 & 2.44  \\
\midrule
Performer & 0.31 & 11.8 & 0.82 & 0.62 & 23.8 & 0.84 & 1.25 & 48.7 & 0.85 & 2.47 & 102.4 & 0.89 & 5.02 & 224.9 & 0.93  \\
 & 0.33 & 13.1 & 0.81 & 0.67 & 26.2 & 0.83 & 1.33 & 53.4 & 0.85 & 2.66 & 111.6 & 0.88 & 5.33 & 243.4 & 0.92  \\
 & 0.37 & 15.6 & 0.81 & 0.75 & 31.0 & 0.83 & 1.50 & 62.7 & 0.85 & 3.00 & 129.9 & 0.88 & 5.98 & 279.7 & 0.92  \\
 & 0.51 & 20.7 & 0.80 & 1.03 & 40.6 & 0.83 & 1.92 & 81.2 & 0.84 & 3.70 & 332.9 & 0.88 & 7.40 & 704.4 & 0.92  \\
 & 0.68 & 32.9 & 0.80 & 1.35 & 64.1 & 0.82 & 2.59 & 128.4 & 0.84 & 5.18 & 264.9 & 0.88 & 11.05 & 568.3 & 0.92  \\
 & - & - & - & 2.02 & 115.7 & 0.82 & 4.05 & 232.0 & 0.84 & 8.12 & 479.9 & 0.88 & 16.45 & 1037.0 & 0.92  \\
 & - & - & - & - & - & - & 7.01 & 439.0 & 0.85 & 14.10 & 909.9 & 0.88 & 28.66 & 1974.1 & 0.92  \\
 & - & - & - & - & - & - & - & - & - & 26.17 & 1769.9 & 0.88 & 53.64 & 3848.7 & 0.92  \\
 & - & - & - & - & - & - & - & - & - & - & - & - & 105.43 & 7597.5 & 0.93  \\
\midrule
Longformer & 0.56 & 19.9 & 0.55 & 1.12 & 80.1 & 0.63 & 2.22 & 84.0 & 0.76 & 4.39 & 178.1 & 0.91 & 8.74 & 396.6 & 1.12  \\
 & 0.57 & 20.8 & 0.40 & 1.14 & 42.1 & 0.48 & 2.29 & 86.6 & 0.62 & 4.54 & 183.2 & 0.77 & 9.01 & 407.2 & 0.97  \\
 & 0.64 & 22.1 & 0.29 & 1.30 & 44.6 & 0.36 & 2.60 & 91.6 & 0.50 & 5.08 & 193.0 & 0.64 & 10.18 & 426.8 & 0.84  \\
 & - & - & - & 1.62 & 49.8 & 0.27 & 3.26 & 101.6 & 0.40 & 6.47 & 212.8 & 0.54 & 12.91 & 466.1 & 0.73  \\
 & - & - & - & - & - & - & 4.63 & 130.4 & 0.29 & 9.26 & 271.4 & 0.43 & 18.56 & 666.9 & 0.62  \\
 & - & - & - & - & - & - & - & - & - & 14.69 & 456.8 & 0.33 & 29.98 & 1989.6 & 0.50  \\
 & - & - & - & - & - & - & - & - & - & - & - & - & 52.06 & 1820.5 & 0.37  \\
\midrule
Big Bird & 1.02 & 30.5 & 0.34 & 2.13 & 63.7 & 0.43 & 4.12 & 132.1 & 0.55 & 10.28 & 276.3 & 0.73 & 17.36 & 595.3 & 0.90  \\
 & 1.05 & 30.5 & 0.18 & 2.21 & 66.5 & 0.27 & 4.49 & 140.2 & 0.41 & 9.89 & 295.4 & 0.57 & 19.69 & 720.0 & 0.76  \\
 & - & - & - & 2.24 & 73.2 & 0.14 & 4.70 & 167.0 & 0.29 & 10.91 & 367.1 & 0.43 & 19.49 & 828.7 & 0.61  \\
 & - & - & - & - & - & - & 5.53 & 203.7 & 0.17 & 12.59 & 466.7 & 0.32 & 24.57 & 1054.3 & 0.48  \\
 & - & - & - & - & - & - & - & - & - & 16.93 & 666.9 & 0.18 & 34.53 & 1630.3 & 0.34  \\
\midrule
Scatterbrain & 0.96 & 40.1 & 0.60 & 1.85 & 80.0 & 0.67 & 3.84 & 161.7 & 0.79 & 7.66 & 332.7 & 0.93 & 15.06 & 786.6 & 1.12  \\
 & 0.95 & 40.1 & 0.45 & 1.91 & 81.4 & 0.52 & 3.81 & 167.0 & 0.65 & 7.66 & 353.6 & 0.79 & 15.42 & 788.2 & 0.97  \\
 & 1.05 & 40.1 & 0.33 & 2.12 & 160.0 & 0.41 & 4.21 & 161.7 & 0.53 & 8.50 & 629.2 & 0.67 & 17.14 & 788.2 & 0.86  \\
 & - & - & - & 2.53 & 160.0 & 0.31 & 5.07 & 167.0 & 0.44 & 10.15 & 629.2 & 0.57 & 20.52 & 788.2 & 0.77  \\
 & - & - & - & - & - & - & 6.73 & 198.6 & 0.33 & 13.53 & 416.8 & 0.47 & 27.36 & 914.6 & 0.66  \\
 & - & - & - & - & - & - & - & - & - & 20.11 & 585.4 & 0.36 & 42.60 & 1251.7 & 0.52  \\
 & - & - & - & - & - & - & - & - & - & - & - & - & 67.81 & 2091.2 & 0.39  \\
\midrule
MRA-2 & 0.32 & 12.6 & 0.44 & 0.65 & 25.5 & 0.51 & 1.28 & 52.2 & 0.61 & 2.57 & 109.6 & 0.72 & 5.18 & 240.3 & 0.87  \\
 & 0.33 & 13.0 & 0.34 & 0.66 & 26.2 & 0.40 & 1.34 & 53.8 & 0.50 & 2.68 & 112.9 & 0.63 & 5.41 & 246.9 & 0.78  \\
 & 0.37 & 14.2 & 0.21 & 0.73 & 28.7 & 0.28 & 1.47 & 58.9 & 0.40 & 2.95 & 123.3 & 0.53 & 5.94 & 268.9 & 0.69  \\
 & - & - & - & 0.88 & 34.9 & 0.15 & 1.75 & 71.3 & 0.28 & 3.50 & 148.1 & 0.41 & 7.23 & 318.5 & 0.58  \\
 & - & - & - & - & - & - & 2.30 & 100.2 & 0.16 & 4.61 & 210.7 & 0.29 & 9.28 & 462.6 & 0.45  \\
 & - & - & - & - & - & - & - & - & - & 6.81 & 377.0 & 0.16 & 13.74 & 836.1 & 0.32  \\
 & - & - & - & - & - & - & - & - & - & - & - & - & 22.68 & 1583.3 & 0.17  \\
MRA-2-s & 0.28 & 10.9 & 0.54 & 0.57 & 22.1 & 0.63 & 1.13 & 45.5 & 0.77 & 2.26 & 96.1 & 0.93 & 4.55 & 212.5 & 1.14  \\
 & 0.30 & 11.3 & 0.40 & 0.60 & 22.9 & 0.49 & 1.19 & 47.1 & 0.61 & 2.38 & 99.4 & 0.78 & 4.78 & 219.1 & 0.98  \\
 & 0.33 & 12.1 & 0.23 & 0.66 & 24.5 & 0.33 & 1.32 & 50.2 & 0.47 & 2.64 & 105.6 & 0.63 & 5.27 & 231.5 & 0.83  \\
 & - & - & - & 0.79 & 29.8 & 0.17 & 1.60 & 60.8 & 0.31 & 3.19 & 126.8 & 0.47 & 6.36 & 274.0 & 0.67  \\
 & - & - & - & - & - & - & 2.15 & 99.8 & 0.17 & 4.31 & 209.9 & 0.32 & 8.61 & 460.7 & 0.51  \\
 & - & - & - & - & - & - & - & - & - & 6.50 & 376.2 & 0.17 & 13.08 & 834.2 & 0.35  \\
 & - & - & - & - & - & - & - & - & - & - & - & - & 22.03 & 1581.2 & 0.19  \\
\bottomrule
\end{tabular}
\end{tiny}
\end{center}
\caption{Efficiency vs Approximation for approximation methods. The units for time and memory are ms and MB respectively. Error is the relative error used in the main text. }
\label{tab:detailed_efficiency_table}
\end{table*}

\textbf{Efficiency}. We provide more details about the efficiency shown in the main text. The efficiency is measured on a single Nvidia RTX 3090. We use the largest possible batch size for each method and average the measurements over multiple steps and batch sizes to get an accurate measurement of runtime and memory consumption of a single instance. We show the efficiency and error for sequence length $256$, $512$, $1024$, $2048$, and $4096$. The results are shown in Tab. \ref{tab:detailed_efficiency_table}

\textbf{FLOPS vs Runtime}. Existing automated profilers only support a limited set of PyTorch operators like Conv2D and BatchNorm (e.g., ptflops, thop, deepspeed, or PyTorch profiler cannot be used for all operators). So, runtime appeared a good indicator of efficiency. 

\textbf{Hyperparameters}. In this section, we provide more details about the experiments. We run all experiments on a $8\times$ Nvidia RTX 3090 server. The hyperparameters of each experiments are summarized in Tab. \ref{tab:hyperparameters}.

\begin{table*}[!htbp]
\begin{center}
\begin{scriptsize}
\begin{tabular}{l|ccccccccc}
\toprule
Experiment & \multicolumn{4}{c}{RoBERTa-base} & \multicolumn{4}{c}{RoBERTa-small}  & ImageNet \\
Task & \multicolumn{2}{c}{MLM} & MNLI & WikiHop & \multicolumn{2}{c}{MLM} & MNLI & WikiHop \\
Sequence Length & 512 & 4096 & 512 & 4096 & 512 & 4096 & 512 & 4096 & 1024 \\
\midrule
Num of Layers    & 12 & 12 & 12 & 12 & 4 & 4 & 4 & 4 & 4 \\
Embedding Dim    & 768 & 768 & 768 & 768 & 128 & 128 & 128 & 128 & 128 \\
Transformer Dim  & 768 & 768 & 768 & 768 & 384 & 384 & 384 & 384 & 128 \\
Hidden Dim       & 3072 & 3072 & 3072 & 3072 & 1536 & 1536 & 1536 & 1536 & 512 \\
Num of Heads     & 12 & 12 & 12 & 12 & 6 & 6 & 6 & 6 & 2 \\
Head Dim         & 64 & 64 & 64 & 64 & 64 & 64 & 64 & 64 & 64 \\
Batch Size       & 512 & 64 & 32 & 32 & 512 & 64 & 32 & 32 & 256 \\
Learning Rate    & 5e-5 & 5e-5 & 3e-5 & 5e-5 & 1e-4 & 5e-5 & 3e-5 & 5e-5 & 5e-4 \\
Num of Steps     & 20K & 75K & 4 epochs & 15 epochs & 150K & 75K & 10 epochs & 25 epochs & 300 epochs \\
\bottomrule
\end{tabular}
\end{scriptsize}
\end{center}
\caption{Hyperparameters for all experiments.}
\label{tab:hyperparameters}
\end{table*}

\subsection{Wavelets}
\label{sec:wavelets}

Multiresolution analysis reorganizes a signal into different strata or resolutions, where the lower resolutions contain information summarizing global features of the signal, and  higher ones capture fine-grained details of the signal.  These strata are constructed iteratively, and we briefly describe the construction.  For simplicity, we will focus on signals that are 1-dimensional (e.g., time series), but all of the following extends to signals of arbitrary dimension.


Construction of the filters $L$ and $H$, which has been the subject of extensive research~\cite{daubechies1992ten}, guarantees that the wavelet analysis operator $W:L_0\rightarrow (L_N, H_N, H_{N-1}, \ldots, H_{1})$ is a linear isometry. That is,
\begin{align*}
\norm{L_0}^{2} = \norm{W(L_0)}^{2} = \norm{L_N}^2 + \sum_{0 < k \leq N} \norm{H_k}^2.
\end{align*}
Thus, the reorganization of a signal under the action of $W$ has a linear inverse, $W^{-1}=W^*$, which is the adjoint of $W$. 


A wavelet decomposition allows one to alter values within particular resolution in an attempt to suppress information that we wish to ignore~\cite{donoho1995noising}. For example, if we wish to ignore fine-grained changes in the data, we would apply a threshold to the higher level, and replace small values in them with zero, and then reconstruct a new signal that has had the higher frequency information purged from it.

There is a tight connection between a signal's smoothness and the energy captured by the coefficients in $H_k$~\cite{donoho1995adapting}. Informally, most of the energy of a smooth function is captured in the low frequency wavelet coefficients, and the energy captured by the high-frequency coefficients decreases as $k$ decreases. Smoothness can be characterized by the rate of decrease. A reasonable heuristic is that  smooth functions have sparse representations, and most energy lies within the lower strata of a wavelet decomposition. However, the smoothness of the wavelet basis also has a role here. The Haar wavelet system is not continuous, and one consequence of this is that energy of smooth functions leaks into higher frequency strata.  The result is that wavelet decompositions based on the Haar system are typically not as sparse as decompositions based on more sophisticated wavelet systems. 

Nevertheless, the Haar system is useful for its simplicity. 
Using the Haar system, we can bound the high frequency wavelet coefficients of $\mathcal{P}=QK^T$ in terms of the corresponding wavelet coefficients of $Q$ and $K$. The  filters $L$ and $H$ of the Haar system are $L\coloneqq (2^{-1/2}, 2^{-1/2})$ and $H\coloneqq (2^{-1/2}, -2^{-1/2})$. 
Column $j$ of $\mathcal{P}$ convolved with $H$ has entry $i$ equal to $2^{-1/2}(\mathcal{P}_{i, j} - \mathcal{P}_{i+1, j})$. If we denote the downsampling operation by $\operatorname{down}$ then entry $i$ of $H_1^{\mathcal{P}_{\cdot,j}}\coloneqq \operatorname{down}(\mathcal{P}_{\cdot,j}*H)$  is $2^{-1/2}(\mathcal{P}_{2i, j} - \mathcal{P}_{2i+1, j})$. Hence,
\begin{equation}
\begin{split}
\norm{H_1^{\mathcal{P}_{\cdot,j}}}_1 &= \sum_{i}\left|2^{-1/2}(\mathcal{P}_{2i, j} - \mathcal{P}_{2i+1, j} )\right| 
= \sum_{i}\left|2^{-1/2}(Q_{2i} - Q_{2i+1}) K_{j}^T\right| \\
&\leq \sum_{i}\norm{2^{-1/2}(Q_{2i} - Q_{2i+1})}_p\norm{K_j}_q 
= \left(\sum_{i} \norm{(H^{Q}_1)_{i,\cdot}}_p\right)\norm{K_j}_q
\end{split}
\end{equation}
where $(H^{Q}_1)_{i,\cdot}\in\R^{d}$ and $1/p+1/q=1$. Informally, the smoothness of the columns of $\mathcal{P}$ is controlled by the smoothness of the columns of $Q$. A nearly identical statement can be made about the smoothness of the rows of $\mathcal{P}$ and the smoothness of the rows of $K^T$. The same argument extends to the other strata of the wavelet decomposition. It is natural to compare this estimate to the bounds used in Lemma \ref{lem:error_bound_of_mu}.



\end{document}